\documentclass{article}

\usepackage{microtype}
\usepackage{graphicx}
\usepackage{subfigure}
\usepackage{booktabs}

\usepackage{hyperref}

\usepackage[accepted]{icml2023}

\usepackage{amsmath}
\usepackage{amssymb}
\usepackage{mathtools}
\usepackage{amsthm}

\usepackage[capitalize,noabbrev]{cleveref}

\theoremstyle{plain}
\newtheorem{theorem}{Theorem}[section]

\theoremstyle{definition}
\newtheorem{definition}[theorem]{Definition}

\theoremstyle{remark}

\usepackage[textsize=tiny]{todonotes}

\usepackage{comment}
\usepackage{hyperref}
\usepackage{url}
\usepackage{textcomp}
\usepackage{gensymb}
\usepackage{amsmath}
\usepackage{amssymb}
\usepackage{mathtools}
\usepackage{amsthm}

\usepackage{thmtools}
\usepackage{thm-restate}
\usepackage{wrapfig}
\usepackage{graphicx}
\usepackage{stmaryrd}
\usepackage{booktabs}
\usepackage{color}
\usepackage{xcolor}
\makeatletter

\definecolor{myred}{RGB}{140, 0, 0}
\definecolor{myblue}{RGB}{0, 27, 155}
\definecolor{myviolet}{RGB}{136, 46, 189}

\newtheorem*{rep@theorem}{rep@title}
\newcommand{\newreptheorem}[2]{%
\newenvironment{rep#1}[1]{%
 \def\rep@title{#2 \ref{##1}}%
 \begin{rep@theorem}}%
 {\end{rep@theorem}}}
\makeatother

\newtheorem{prop}{Proposition}
\newreptheorem{prop}{Proposition}

\declaretheorem[name=Assumption]{assum}

\usepackage{hyperref}

\newcommand{\RR}{\mathbb{R}}

\icmltitlerunning{Homomorphism Autoencoder}

\begin{document}

\twocolumn[
\icmltitle{Homomorphism Autoencoder --- Learning Group Structured Representations from Observed Transitions}

\icmlsetsymbol{equal}{*}

\begin{icmlauthorlist}
  \icmlauthor{Hamza Keurti}{mpi,ini,cls} \icmlauthor{Hsiao-Ru Pan}{mpi}  
  \icmlauthor{Michel Besserve}{mpi}
  
  \icmlauthor{Benjamin Grewe}{ini} \\
  \icmlauthor{Bernhard Sch\"olkopf}{mpi}
\end{icmlauthorlist}

\icmlaffiliation{mpi}{Max Planck Institute for Intelligent Systems, T\"{u}bingen, Germany}
\icmlaffiliation{ini}{Institute of Neuroinformatics, ETH Z\"urich, Switzerland}
\icmlaffiliation{cls}{Max Planck ETH Center for Learning Systems}

\icmlcorrespondingauthor{Hamza Keurti}{hamza.keurti@tuebingen.mpg.de}

\icmlkeywords{Machine Learning, ICML}

\vskip 0.3in
]

\printAffiliationsAndNotice{} 

\begin{abstract}
How can agents learn internal models that veridically represent interactions with the real world is a largely open question. As machine learning is moving towards representations containing not just observational but also interventional knowledge, we study this problem using tools from representation learning and group theory. 
We propose methods enabling an agent acting upon the world to learn internal representations of sensory information that are consistent with actions that modify it.
We use an autoencoder equipped with a group representation acting on its latent space, 
trained using an equivariance-derived loss in order to enforce a suitable homomorphism property on the group representation. 
In contrast to existing work, our approach does not require prior knowledge of the group and does not restrict the set of actions the agent can perform. 
We motivate our method theoretically, and show empirically\footnote{Code can be found at \url{https://github.com/hamzakeurti/homomorphismvae}} that it can learn a group representation of the actions, thereby capturing the structure of the set of transformations applied to the environment. We further show that this allows agents to predict the effect of sequences of future actions with improved accuracy. 
\end{abstract}

\section{Introduction}
One of the most enigmatic questions addressed by mammalian intelligence is how to build an internal model of the external world that represents all behavior-relevant information to efficiently anticipate the consequences of future actions. 
Humans acquire such internal models by interacting with the world, but the learning principles allowing it remain elusive. 
We investigate how Machine Learning (ML) can shed light on this question, as it moves towards representations that carry more than just observational information \citep{Sutton1998,Scholkopfetal21} and develops tools for interactive and geometric structure learning \citep{CohenW16a, AliEslami2018}, 
 
Our setting is inspired by neuroscientific evidence that, as animals use their motor apparatus to act, efference copies of motor signals are sent to the brain's sensory system where they are integrated with incoming sensory observations to predict future sensory inputs \citep{Keller2012}. 
We argue that such efference copies can be useful for learning {\em structured} latent representations of sensory observations and for disentangling the key latent factors of behavioral relevance. 
This view is also in line with hypotheses formulated by developmental psychology \citep{Piaget1964a}, stating that perceiving an object is not creating a mental copy of it but rather internalizing an understanding of how this object transforms under different interventions. 
We translate this idea into an artificial setting in which an agent has to build an internal representation of a novel environment through an interactive process. 
To intervene on its environment, the agent is allowed to perform a set of transformations, while observing the impact of its actions. The agent is provided with a value of the efference copy associated with each transformation, but does not have prior knowledge on how they are structured and how they affect the environment.

In reinforcement learning (RL), interaction probes the structure of the environment and can be leveraged to learn predictive world models that capture the effect of actions \citep{ha2018world}. 
While RL heavily relies on a reward signal, this does not cover all forms of human learning, and other objectives compatible with biological evidence have been proposed, such as \textit{predictive coding} \cite{rao1999predictive,pezzulo2022evolution}. Our approach, too, aims to complement rather than utilize RL. 

We will aim to learn a type of correspondence (mathematically, a homomorphism) between the interventional structure of the world and our representation thereof. To this end, we make the assumption that an agent's actions form a group (or a subset thereof) where actions are composable. 
Further, for learning the environmental transformations caused by its actions, the agent needs to learn a sensory latent space consistent with its observations. 
Learning action-observation relationships then amounts to predicting the change in the agent's latent sensory space using the incoming motor commands. 
Mathematically, this corresponds to an equivariance property of the sensory representations with regard to the group acting on the environment \citep{Dodwell1983}.

While one can mathematically create many such representations, an agent with bounded computational abilities needs to choose one allowing efficient manipulation, interpretation and prediction of changes in its environment. Group theory suggests that transformations can be efficiently encoded as matrices, notably through the concept of \textit{group representation}. In an appropriately chosen latent space, transformations of world states can be efficiently encoded as matrix-vector products, and composed using matrix multiplication.
Additionally, a representational property favoring efficiency is disentanglement \citep{Bengio2012,kulkarni2015deep,Scholkopfetal21}, 
which states that the latent representation can be decomposed into subspaces reflecting properties of the environment that the agent can modify independently. Disentanglement can be framed in a group theoretic setting by imposing a block diagonal structure on the group representation \cite{Higgins2018}, thereby promoting a parsimonious encoding of the group structure.

While mathematical results state that such a latent representation, where group elements act as matrices, exists under mild conditions on the group action \citep{antonyan2009linearization,kraft2014families}, how to learn it from data without prior knowledge of the group structure is an open question. 
In the present work, we investigate this question when observation-action sequences carried out by an agent embedded in an environment can be observed. 
Our approach only makes minimal assumptions with respect to the latent structure to be learned. 
In particular, instead of directly enforcing a parametrized representation of a particular group, we allow for arbitrary mappings to the space of linear maps.
Our mapping is naturally shaped into a group representation by observing how actions affect future percepts.
We propose a training loss derived from the commutative diagram that an equivariant map satisfies, and a predictive autoencoder termed {\em Homomorphism AutoEncoder (HAE)} trained on our equivariance constraints.
We also propose a sparsity regularization which favors disentangled representations.

Our contributions can be summarized as follows:
\vspace*{-.1cm}
\begin{itemize}
\itemsep0em
\parskip0em
    \item We propose the HAE framework to jointly learn $(\rho,h)$ a group representation of transitions, as well as a symmetry-based disentangled representation of observations without prior knowledge on the group. 
    
    \item We provide theoretical justification and experimental evidence that the HAE learns the group structure of the transitions for different groups. 
    \item The HAE learns to separate the pose of an object with regard to the transitions group acting on it, from the identity of acted-on objects as orbits of the non-transitive group action.
    \item We show performance exceeding previous approaches despite not using prior knowledge on the group structure to be learned. 
\end{itemize}

\paragraph{Related Work.}

Learning useful representations has been addressed with a wide range of methods. Given only unlabelled observational data, generative models 
have difficulties to infer ground truth disentangled factors without additional assumptions on the mechanisms \citep{Locatello2019}. 
In the i.i.d.\ data setting, constraining the function class mapping the true latent factors to the observations can help identify these factors \citep{gresele_independent_2021}. 
Access to non-i.i.d.\ data, e.g., from multiple environments, has been shown to allow uncovering ground truth latent factors, notably through contrastive learning approaches \citep{hyvarinen_unsupervised_2016,khemakhem_variational_2020,von_kugelgen_self-supervised_2021}.  
This is also in line with methods relying on interventional data to provide more information about representational structure. 

For instance, \citet{sontakke2021causal} propose disentangling binary factors of variations by learning interactions which best separate sensory trajectories. 
\citet{Thomas2017} propose a specific objective function that leads agents to learn policies that separately act on the disentangled properties of the environment.

Homomorphic MDPs \cite{vanderpol20b} propose equivariant abstractions of states based on the similarity of both transitions and rewards.

There is as well a growing interest in group theory to enforce structure in learned representations. 

Group invariance has been suggested as a useful inductive bias to learn causal structure \citep{besserve2018aistats,Besserve_Sun_Janzing_Scholkopf_2021}. 
While CNNs \citep{LeCun1989} leveraged the linear action of the translation group on the image space, 
this idea was generalized to other groups \citep{CohenW16a} or on other data modalities, 
leading to the design of appropriate group equivariant linear maps, and theoretical characterizations 
\cite{Kondor2018}.

\citet{Dehmamy2021,Zhou2020,Finzi2021} additionally propose to learn group-equivariant maps from the data instead of imposing a particular convolution pattern. 
A limitation of these works is that they learn symmetries encoded as linear (matrix) actions in the measurement space, while many relevant transformations of the world can only take such form in a well chosen latent representation of it, and the non-linear mapping between this representation and the observations typically needs to be learned as well. 
\citet{Higgins2018} propose a theoretical framework for such \textit{group-structured representation} to propose a mathematical definition of disentanglement. 

Several works describe approaches to learn these group-structured representations  \citep{Tao2022,Caselles-Dupre2019,Quessard2020,park2022learning}, to assess them \citep{Tonnaer2020} and to infer transformations between pairs of observations \citep{Painter20,Connor2020}. 
Inspiring our approach, \citet{Caselles-Dupre2019,Quessard2020} derive optimization losses from equivariance requirements, and \citet{Connor2020} leverage the Lie Algebra and the matrix exponential. 
All these works constrain the transformations observed during training to be restricted to certain subgroups, 
thereby assuming the agent has a preference for separately acting on specific aspects of the state. 
With the exception of \citet{Quessard2020}, all these works also constrain the group to be a product of cyclic groups. 
In contrast, our work does not require prior knowledge on the group structure, nor does it

restrict the set of allowed transformations between pairs of observations.

\section{Background}\label{sec:background}
\begin{figure}[t]
    \vspace{0pt}
    \centering
    \includegraphics[width=0.45\linewidth]{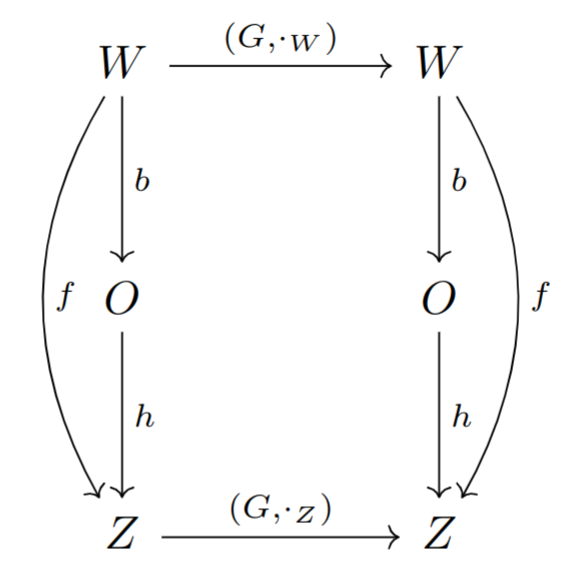}
    \includegraphics[width=0.35\linewidth]{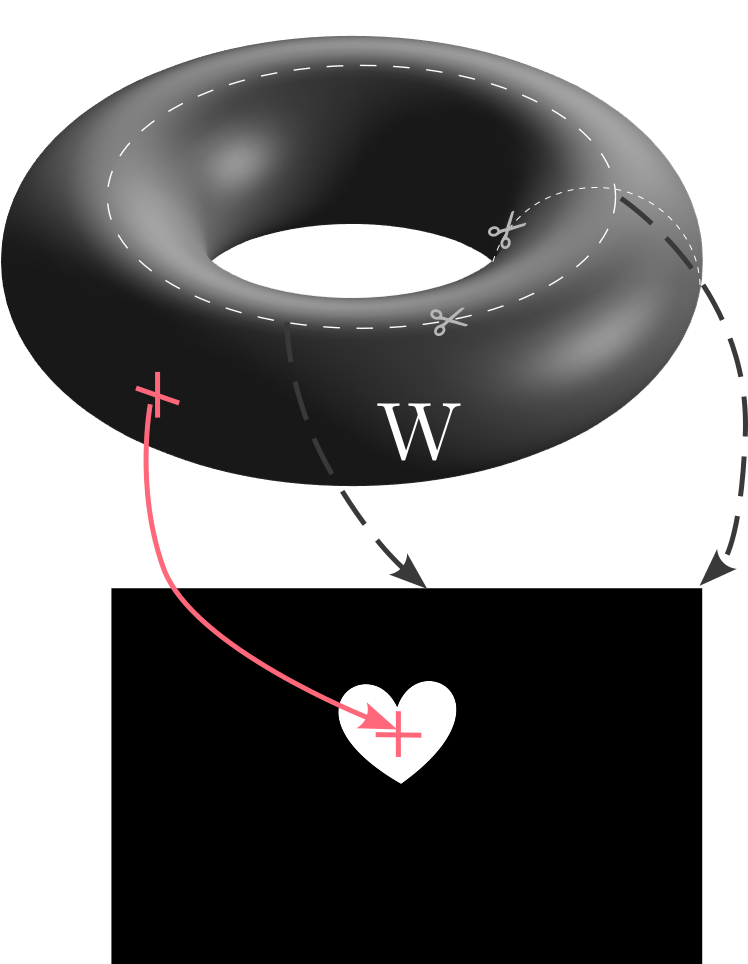}
    \vspace{-5pt}
    \caption{\textbf{Left:} Commutative diagram of a group-structured representation. \textbf{Right:} Toroidal latent world for a moving 2D object observation, parameterized by the coordinates of the objects' center (red). Dashed lines describe the unwrapping process. 
    }
    \vspace{-.5cm}
    \label{fig:comdiag1}
\end{figure}
\subsection{Symmetry-Based Disentangled Representation Learning}\label{sec:sbdrl}
We refer to Appendix~\ref{app:grp} for background on groups. We assume the set of observations $O\subset \RR^{n_o}$ is obtained from a space of world states $W$ through an unseen \textit{generative process} $b:W\rightarrow O$.
For example, consider translations of the heart-shaped 2D object on an image represented in Figure~\ref{fig:comdiag1}. If we allow translations leaving the image to move the object to the opposite side, $W$ corresponds to a representation of the space of positions of the heart's center on a torus, that can be identified with $[0,1]\times [0,1]$, and we observe pixel images $o\in O = \RR^{64\times 64}$ (see~Figure~\ref{fig:approach}) of this torus by unwrapping it, opposite borders of this image coincide on the original torus.
An \textit{inference process} $h:O\rightarrow Z$ maps observations to their vector representations $z\in Z=\RR^D$.

A group of symmetries $G$ structures the world states by its action $\cdot_W:G\times W \rightarrow W$. $G$ is decomposed into a direct product of subgroups $G = G_1 \times ... \times G_{n}$. Here, each subgroup only transforms a specific latent property while keeping all others constant.
In the example, the 2D toroidal group $G \simeq SO(2)\times SO(2)$ acts on the heart by translating it on the torus through the action: $ g_{\theta_1,\theta_2}\cdot_W~w~=~(\frac{\theta_1}{2\pi}~+~w_1,\frac{\theta_2}{2\pi}+w_2) \mod 1$. 
These transformations are reflected in the image space by a shift of pixel activations (see Figure~\ref{fig:approach}) until it reaches the border after which it appears on the other side.

We call a representation {\em group-structured} (cf. \citealt{Higgins2018})

if it satisfies the commutative diagram in Figure~\ref{fig:comdiag1} such that:
\begin{enumerate}
    \item There is an action of $G$ on $Z$: $\quad
            \cdot_Z: G\times Z \rightarrow Z$.
    \item The composition $f = h \circ b: W \rightarrow Z$ is equivariant, meaning that transformations of $W$ are reflected on $Z$, i.e., $\forall g\in G, w\in W,\quad f(g\cdot_W w) = g\cdot_Z f(w).$
\end{enumerate}

A group action $\cdot_Z:G\times~Z\rightarrow~Z$ induces a group homomorphism $\rho:G\rightarrow Sym(Z)$ where $Sym(Z)$ is the group of invertible mappings from $Z$ to itself (more in Appendix~\ref{app:grp}). 
We require the group action $\cdot_Z$ on $Z$ to be linear, in which case the induced homomorphism $\rho:G\rightarrow GL(Z)$ is called a group representation (here, $GL(Z)$ is the group of invertible linear maps on the vector space $Z$).

We call a group-structured representation {\em disentangled} with regard to the group decomposition $G=G_1\times ... \times G_n$ if it satisfies this additional condition:
\begin{enumerate}
    \setcounter{enumi}{2}
    \item  There exists a decomposition $Z = Z_1 \oplus ... \oplus Z_n$ and a decomposition of the group representation $\rho = \rho_1 \oplus ... \oplus \rho_n$ where each $\rho_i:G_i \rightarrow GL(Z_i)$ is a subrepresentation.
\end{enumerate}

The action on $Z$ can then be written \begin{eqnarray}
    g\cdot_Z z &=& \rho(g_1,...,g_n)(z_1 \oplus ... \oplus z_n) \nonumber\\ &=& \rho_1(g_1)z_1 \oplus ... \oplus \rho_n(g_n)z_n , 
\end{eqnarray}
for $g = (g_1,...,g_n)\in G$  and $z = z_1 \oplus ... \oplus z_n \in Z$.
Clearly, each subgroup $G_i$ acts trivially on $Z_j,\,j\neq i$.

By choosing the mapping $h$ such that $f=h\circ b$ satisfies $f(w)~=~(\cos(2\pi w_1),\sin(2\pi w_1),\cos(2\pi w_2),\sin(2\pi w_2))^T$, and the block-diagonal group representation $\rho$ such that:
\begin{equation}\label{eq:expectedrho}
\rho=\rho_1\oplus \rho_2; \quad
\rho_i:g_{\theta_1,\theta_2}\mapsto \begin{pmatrix}
    \cos(\theta_i) & -\sin(\theta_i)\\
    \sin(\theta_i) & \cos(\theta_i) \\
\end{pmatrix}\,,
\end{equation}
we obtain a linear group-structured representation $(\rho,h)$ disentangled with regard to a decomposition of the group acting on the latent 2-torus. In Section~\ref{sec:exp_torus}, we learn such a representation using the HAE described in Section~\ref{sec:hae}.

\subsection{Lie Groups and the Exponential Map}\label{sec:expmap}

A {\em Lie group} is a group $G$ that is also a finite-dimensional differentiable manifold with smooth composition and inverse.
Its tangent space at the identity is called a {\em Lie Algebra} $\mathfrak{g}$.
An interesting property of Lie groups is the existence of the \textit{exponential map} which lets us generate elements of $G$ from those of $\mathfrak{g}$. 
Restricting ourselves to connected compact \textit{matrix Lie Groups}, the exponential map becomes surjective; it is defined by the series $\exp(A) = \sum_{k=0}^\infty \frac{1}{k!} A^k$.

We leverage this link between the Lie Group and its Algebra to use a parametrization of group representations 
${\rho:G\to GL(Z)}$ 
as the composition 
${\rho = \exp \circ \phi}$ of ${exp:\mathfrak{gl}(Z)\to GL(Z)}$ that maps square matrices to invertible matrices and 
${\phi:G\to \mathfrak{gl}(Z)}$ which maps transition signals to square matrices. This parametrization of $\rho$ allows flexibility for the group $G$ and provides optimization advantages. We detail the derivation of this composition as well as its advantages in the Appendix~\ref{app:rhophi}.

\section{Approach and Experimental Setup}\label{sec:hae}
\begin{figure}[ht]
    \vspace{-4pt}
    \centering
    \includegraphics[width=0.95\linewidth]{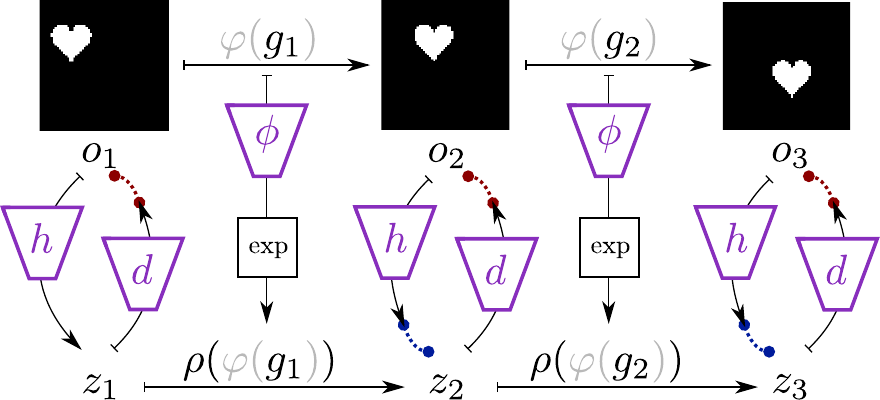}
    \vspace{-5pt}
    \caption{The Homomorphism Autoencoder (HAE) consisting of $h$ (encoder), $d$ (decoder) and $\rho=\exp \circ \, \phi$ (group representation) relies on 2-step latent prediction to jointly learn the group representation $\rho$ and the observation representation $h$. The HAE learns by jointly minimizing both the \textbf{\textcolor{myblue}{latent prediction loss}} and the \textbf{\textcolor{myred}{reconstruction loss}} (dotted connections) to simultaneously learn representations of the observations and the group actions.}
    \vspace{-4pt}
    \label{fig:approach}
\end{figure}

\subsection{The Agent Observes Structure through Interaction}

In order to discover the underlying structure relating observations, we consider an agent that interacts with its environment and produces rollouts $(o_1, g_1, o_2, g_2,..., g_{N-1},o_N)$ of observations and group elements which transform the underlying latent state.
The agent does not directly have access to the group elements $g$, but to unstructured efference copies $\varphi(g)$ of its performed actions,
which, we assume, are related to $g$ through an unknown fixed injective mapping $\varphi$. 
We use a random neural network for $\varphi$ in the experiments.
To avoid cluttering, we write $g$ instead of $\varphi(g)$ in the rest of the paper.

Unlike other works \citep{Quessard2020,Painter20,Caselles-Dupre2019}, we do not assume the agent knows how to act specifically along subgroups of the decomposition.
The agent explores its action-perception dependencies through a random policy, where it samples actions from a neighbourhood of the identity. 
Finally, the agent does not have knowledge of the composition rules of the group element and only witnesses side effects of composition through the succession of observations.

Through our proposed Homomorphism Autoencoder, the agent will learn a group-structured representation $(\rho,h)$ that is disentangled with respect to a decomposition of the underlying group $G$.

\subsection{Homomorphism Autoencoder (HAE) Architecture}

To jointly learn the latent representation $h$ of the observations and the group representation $\rho$, we introduce the HAE, described in Figure~\ref{fig:approach}. This is a deterministic autoencoder with encoder $h$ and decoder $d$, endowed with a group representation $\rho:G\rightarrow GL(Z)$ parametrized by a neural network $\phi$ as described in Section~\ref{sec:expmap}. $\rho=\exp\circ \phi$ acts on encoding vectors of observations $z_t = h(o_t)$ to predict the encoding of future stimuli through ${\rho(g_t)\cdot z_t \approx z_{t+1}}$. 
The latent prediction is evaluated on both the latent space through the $2$-step latent prediction loss $\mathcal{L}^2_{pred}$ and on the image space through the $2$-step reconstruction loss $\mathcal{L}^2_{rec}$. We optimize $\mathcal{L}$, a weighted sum of both losses, setting their relative importance with coefficient $\gamma$, yielding equation~\ref{eq:losses}. 
The losses are separately detailed in the section \ref{sec:theory}.
\begin{equation}\label{eq:losses}
    \mathcal{L}(\rho,h,d) = \mathcal{L}^2_{rec}(\rho, h,d) + \gamma  \mathcal{L}^2_{pred}(\rho,h)\,.
\end{equation}

\section{The HAE Learns Symmetry-Based Representations}\label{sec:theory}

Previous attempts to design symmetry-based disentangled linear representations have put a lot of emphasis on the disentanglement property. 
However, it remains unclear how to learn a symmetry-based linear representation $(\rho,h)$ that verifies properties 1 and 2 in Section~\ref{sec:sbdrl}, without enforcing strong assumptions on $\rho$ \citep{Caselles-Dupre2019} or on the actions the agent can perform \citep{Caselles-Dupre2019,Quessard2020}. 

In this section, we provide theoretical insights on learning symmetry-based representations and how the two-step HAE architecture achieves it with minimal assumptions.
We define the losses used throughout for a given sample path $(o_1,g_1,\dots,o_N)$.

\textbf{The $N$-step latent prediction loss} 
compares the evolution of stimuli encodings predicted by using the group representation action against the encodings of the corresponding observations (blue dotted connection in Figure~\ref{fig:approach}):

\[
\mathcal{L}_{pred}^N (\rho,h)\!=\!\! \sum_{t=2}^{N+1}  
\left\|h(o_{t}) \!-\! \Big(\prod_{i=1}^{t-1}\rho(g_{i})\Big)h(o_1)\right\|_2^2 
\]

\textbf{The $N$-step reconstruction loss} 
compares the reconstructions of the stimuli obtained from decoding the predicted evolution of encodings  against the actual stimuli (red dotted connection in Figure~\ref{fig:approach}).
The reconstruction loss also evaluates the reconstruction of the initial observation like a standard autoencoder,
\[
\mathcal{L}_{rec}^N(\rho,h,d)\! = \!\! \sum_{t=1}^{N+1} 
\left\|o_{t} \!-\! d\left( \!\Big(\prod_{i\geq 1}^{t-1}\rho(g_{i})\Big)h(o_1)\!\right)\right\|_2^2\,,
\]
where by convention the empty product for $t=1$ is 1. 

In both losses, the Euclidean norm may be replaced by other positive functions. In practice, we use the Euclidean norm for $\mathcal{L}_{pred}$, but use binary cross-entropy for $\mathcal{L}_{rec}$, which is a common choice for image data.
The $1$-step latent prediction loss $\mathcal{L}_{pred}^1$ is simply enforcing the commutative diagram in Figure~\ref{fig:comdiag1}, and $N$-step losses allow us to extend the commutative diagram to multi-step settings.

To provide guarantees for our approach, we make minimal assumptions on the world states symmetries. 
\begin{assum}[Standard action on world states]\label{assum:linearworld}
    $G$ and $W$ are compact smooth manifolds and there exists a map $m:W\to W^*$diffeomorphic onto its image, where $W^*$ is a finite dimensional real vector space such that $G$ admits a continuous injective group representation $\rho^*: G\to GL(W^*)$ and the action of $G$ on $W$ corresponds to the matrix-vector multiplication by $\rho^*$ on $W^*$: $g\cdot_W w =m^{-1}\left(\rho^*(g) m(w)\right)$. 
\end{assum}
This technical assumption encompasses a wide range of transformations of the physical world that can be modelled by linear actions of matrix groups such as $SO(n)$ in an appropriate latent space. It should {\em not} be misunderstood as making any linear approximation of non-linear functions. In particular, the generative process $b$ can be highly non-linear. Compactness allows defining a unique ``uniform'' probability measure on the group and to make sure all world states can be explored with suitable sample paths. Introducing the mapping to $W^*$ allows modelling general linear actions on $W$ as matrix vector products.\footnote{E.g., the conjugate action $MAM^\top$ of matrix $M$ on matrix $A$ can be represented by the vectorization operation $vec$ and the Kronecker product $\otimes$, such that $vec(MAM^\top)=(M\otimes M) vec(A)$ 
} Mild sufficient conditions under which Assumption~\ref{assum:linearworld} holds are given in \citep{antonyan2009linearization,kraft2014families}. 

\paragraph{Theoretical results (see Appendix~\ref{app:proofs}).} If we assume such $\rho^*$ is given, then minimizing $\mathcal{L}_{pred}^1 (\rho^*,h)$ is enough to learn a symmetry-based representation (see Proposition~\ref{prop:gmap} in Appendix~\ref{app:proofs}).

However, when $\rho^*$ is not known, a group representation $\rho$ of $G$ needs to be learned over a space of arbitrary mappings, minimizing $\mathcal{L}_{pred}^1 (\rho,h)$ and can lead to the trivial representation (see Proposition~\ref{prop:trivial} in Appendix~\ref{app:proofs}).

Our main result justifies the use of a 2-step latent prediction loss, to which the observations reconstruction loss should be added, to provide guarantees for the HAE to learn a symmetry-based representation. The proof is in Appendix~\ref{app:proofs}.

{
\renewcommand{\theprop}{\ref{prop:twosteps}}

\begin{prop}[informal]
Under generative model of Section~\ref{sec:sbdrl} with $b$ diffeomorphic onto its image and Assumption~\ref{assum:linearworld}, consider a setting where sample paths have a strictly positive density on a $G$-invariant support. 
If $(\rho,h,d)$ are continuous and minimize the expectation of $\mathcal{L}_{pred}^2(\rho,h) + \gamma \mathcal{L}_{rec}^k(\rho,h,d)$, for $k\geq 0$, then $\rho$ is a non-trivial group representation and $(\rho,h)$ is a symmetry-based representation. 

\end{prop}
\addtocounter{prop}{-1}
}
The reconstruction loss $\mathcal{L}_{rec}^{k}$, $k\geq 0$ is added to the objective to prevent the representation to collapse into a trivial solution by ensuring $h$ is not constant for a given fixed group representation $\rho^0$ (see Proposition~\ref{prop:inject} in appendix). However, while enforcing the reconstruction of only the initial observation using $\mathcal{L}_{rec}^{0}$ is sufficient in theory, we found empirically that using $\mathcal{L}_{rec}^2$ 
performs better when jointly learning $(\rho,h)$.

\section{Learning a Disentangled Representation}\label{sec:framework}

While we previously elaborated on enforcing a symmetry-based representation, we now discuss how to enforce its disentanglement.

As expressed in Section~\ref{sec:sbdrl}, the disentanglement condition for a linear action on $Z$ with group representation $\rho$ is a decomposition of both the representation space $Z=\bigoplus_1^n Z_i$ and the group representation $\rho = \bigoplus_{i=1}^n \rho_i$. 
Here, the subgroup representations $\rho_i$ are representations of the subgroups $G_i$ on the subspaces $Z_i$. 
Following that the group $G$ is decomposed along the latent world's parametrization, the group representation $\rho$ of the action signals $\varphi(g)$ is disentangled with regard to the group decomposition whenever, 
in matrix form, it is a block-diagonal matrix of the subgroups representations:

\begin{equation}\label{eq:blockrho}
    \rho\big(\varphi(g^1,...,g^n)\big) \!=\! \rho_1(g^1) \oplus \rho_2(g^2)\oplus \dots\oplus \rho_n(g^n)\,.
\end{equation}
However, an agent acting on its environment observes transition $\varphi(g)$ but does not necessarily have access to its ground truth decomposition into components $(g^1,...,g^n)\in G_1 \times ... \times G_n$. Therefore, we investigate inductive biases that can constrain our trainable group representation in the space of matrices of the block diagonal form given in equation \ref{eq:blockrho}. 

\subsection{Strong Constraint to Enforce Disentanglement}\label{sec:dis-grp-repr}
We first assume prior knowledge of (1) the number of groups in the decomposition and (2) the dimension of each subgroup representation $dim(Z_i)$, and shape the output of our neural network $\phi$ as in equation~\ref{eq:learnedrho}, its matrix exponential will have the same shape. 
\begin{equation}\label{eq:learnedrho}
    \!\!\phi(\varphi(g)) \!=\!
    \begin{pmatrix}
    \phi_1(\varphi(g)) &     0       & \hdots    &0      \\
    0         & \phi_2(\varphi(g))  & \ddots    &\vdots \\
    \vdots    &\ddots       & \ddots    &0       \\
    0         &\hdots       & 0         & \phi_n(\varphi(g))\\
    \end{pmatrix}
\end{equation}

While we do not prove that this block diagonal constraint leads to disentanglement (each block is a function of the whole group element $g$ and not just the components $g^i$), we show through experiments that the HAE learns a symmetry-based representation $(\rho,h)$ that is disentangled with regard to the ground truth decomposition of the group $G$ and takes the form of equation~\ref{eq:blockrho}. 

\subsection{Soft Constraint to Enforce Disentanglement}\label{sec:soft-dis}
To drop the assumption of prior knowledge of the number of disentangled subgroups and their dimensions, we take inspiration from work on structured sparsity inducing losses \citep{Bach2011} to add the following regularization term in the loss:
\begin{equation}\label{eq:softblock}
    \mathcal{L}_{sparse}(\rho) \!=\! \sum_t \sum_{i\geq 0} \sqrt{\! \sum_{j\geq i+1,\,k\leq i}\!\!\!\!\!\rho_{kj}(g_t)^2 \!+ \!\rho_{jk}(g_t)^2 }
\end{equation}
This term favors block-diagonal patterns (see example in Figure~\ref{fig:softblock}) by jointly minimizing the group of terms $(\rho_{ij}(g_t))_{ij}$ that violate a given block-diagonal configuration. See Appendix~\ref{app:softblock} for more details. 
The model is then trained on the composite loss
\begin{equation*}
    \mathcal{L} = \mathcal{L}_{rec} + \gamma  \mathcal{L}_{pred} + 
    \delta 
    \mathcal{L}_{sparse}\,.
\end{equation*}
We can then choose a representation space $Z=\RR^D$ with $D$ large enough to accommodate the total dimension of the group representation, extra dimensions being ``trivialized'' into one dimensional subspaces by $\mathcal{L}_{sparse}$.

\section{Experiments}
\subsection{Learning a 2D-Toroidal Latent Structure}\label{sec:exp_torus}
We consider a subset of the dSprites dataset \citep{dsprites17} where a fixed scale and orientation heart is acted on by the group of 2D cyclic translations $G = C_x \times C_y$, resulting in the setup described in Section~\ref{sec:background}, where $G$ is a discrete subgroup of the Lie group $SO(2)\times SO(2)$. The group $G$ is decomposed such that the components of its element $g=(g^x,g^y)$ produce a translation on the observed image along the $x$ and $y$ axis, respectively. The HAE is trained on 2-step paths $(o_1,\varphi(g_1),o_2,\varphi(g_2),o_3)$, where $\varphi$ is a randomly chosen non-linear mapping to a $50$-dimensional ambient space, implemented with a neural network. We use the soft constraints for disentanglement.
Architecture and hyperparameters for training are specified in the appendix~\ref{app:exp-single-shape}.

\begin{figure}[ht]
      \vspace{0pt}
    \centering
     \includegraphics[width=0.7\linewidth]{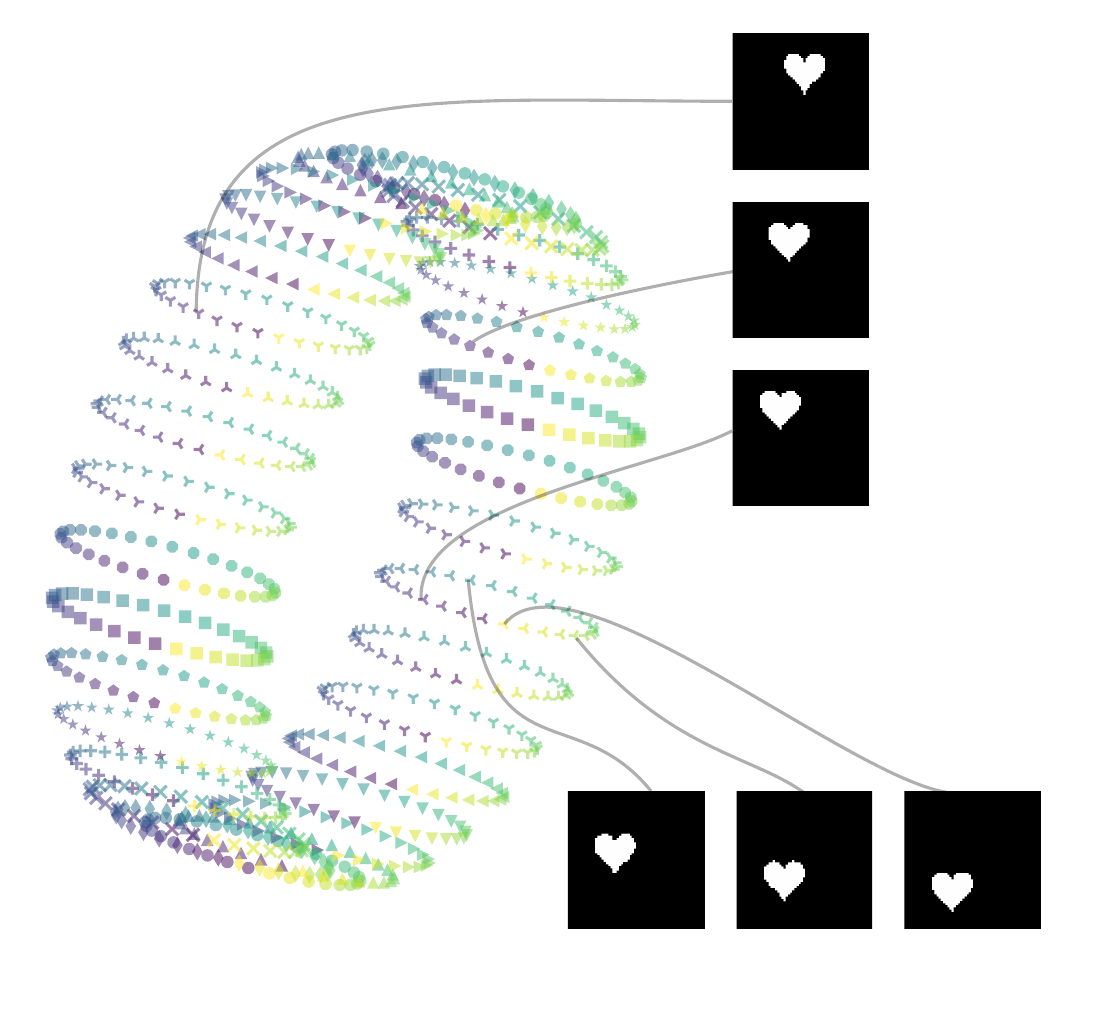}     
     \includegraphics[width=0.6\linewidth]{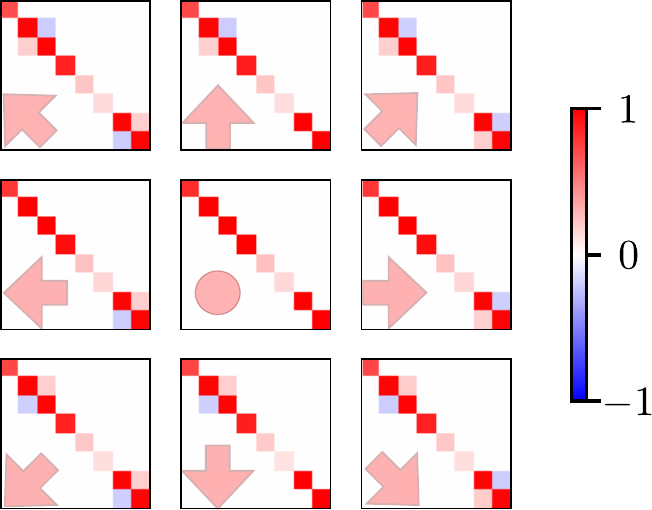}
    \caption{ \textbf{Top:} Projection of the HAEs' $8-$dimensional latent representation vectors $z$ of the translated heart dataset, exhibiting the 2D toroidal structure of the world's state. Color indicates the heart's $x$ position, while markers indicate $y$ position. 
    \textbf{Bottom:} Evaluation of the learned and disentangled group representation $\rho$ for actions over a grid centered on the identity element. Arrows indicate the direction of actions $g$. The representation trivializes the subspace spanned by the indices $1,3,4,5$, $C_x$ and $C_y$ act respectively on dimensions $[2,3]$ and $[7,8]$ through rotation matrices. 
    }
    \label{fig:matrices}    
    \label{fig:manifold}
\end{figure}

\paragraph{Learned data representation.}
We visualize the learned $8-$dimensional encodings of the considered dataset through 2-dimensional random matrix projection (Figure \ref{fig:manifold} top).
The learned manifold corresponds to the expected latent space topology of the 2D torus $S^1\times S^1$ introduced in Section~\ref{sec:background}.

\paragraph{Learned Group Representation $\rho$.}
We then evaluate the learned matrices for the identity $id =(0, 0)$, and a grid of actions around the identity including the generator elements of each subgroup ${1}_x = (\frac{2\pi}{32},0), {1}_y = (0,\frac{2\pi}{32})$ and their inverses $-{1}_x = (-\frac{2\pi}{32},0), -{1}_y = (0,-\frac{2\pi}{32})$. Recall that the components of the group element $g=(g^x,g^y)$ were ``mixed'' using an unknown high-dimensional non-linear mapping $\varphi$  such that axis information of the transformations was not directly accessible. 
The results, Figure~\ref{fig:matrices} (bottom), show that $\rho(id) = \mathbf{I}_8$ as expected from a group representation, and that the representation of generators along each axis acts on a different subspace of the representation. The matrices also correspond to the rotation matrices predicted in Equation~\ref{eq:expectedrho}. A visualization of $\rho$ over a wider neighbourhood of the identity, provided in Figure~\ref{fig:moreactions} of the appendix, shows that composition of transformations are also learned correctly, according to the homomorphism property $\rho(\varphi(gg')) = \rho(\varphi(g))\rho(\varphi(g'))$.

\paragraph{Group representation through the Lie Algebra allows for linear latent traversals.} We describe in Appendix~\ref{app:exp-alg} how our learned group representation $\rho=\exp \circ \, \phi$ gives access to the Lie algebra of the group, which then offers a linear basis $(A_1, A_2)$ to navigate the group and therefore the data manifold (see Figure~\ref{fig:traversal}).

\begin{figure}[ht]
    \centering
    \includegraphics[width=0.45\textwidth]{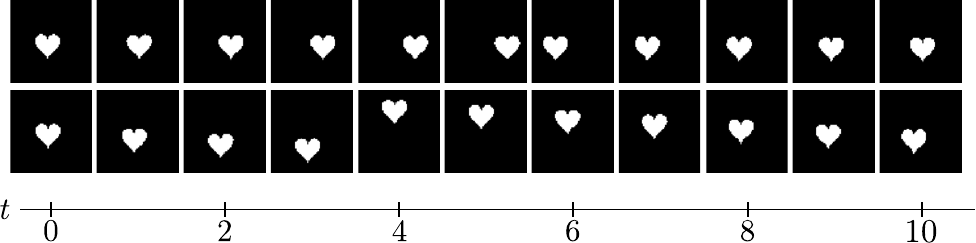}
    \caption{We visualize the linear traversal of the group algebra for the dSprites experiment and its effect on the predicted image reconstruction. The first row corresponds to a traversal $tA_1$ (horizontal displacement), while the second row corresponds to the traversal $tA_2$ (vertical displacement).}
    \label{fig:traversal}
\end{figure}

\paragraph{Learning $SO(2)\times SO(2)\times SO(2)$.} We use the hard constraints for disentanglement presented in section~\ref{sec:dis-grp-repr} to learn a group structured representation for a similar setup where we add the rotations of the heart resulting in a discrete subgroup of $SO(2)\times SO(2)\times SO(2)$. Visualizations and architecture are available in Appendix~\ref{app:add-exp}.

\subsection{Rollout Prediction}

One important application of learning structured representations is to predict how the observations would change given sequences of actions.
We consider longer rollouts $(o_1,g_1,...,g_{N-1},o_N)$ of the dsprites transition dataset the HAE model was trained on.
We compare the HAE to two other approaches of modeling the dynamics in the latent space:
(1) \emph{Unstructured}: $z_{t+1}=h(z_t, a_t)$, where $h$ is a learnable function. 
Similar approaches have been widely adopted in recent model-based deep RL methods \citep{ha2018world, schrittwieser2020mastering}.
(2) \emph{Givens}: $z_{t+1}=R_a z_t$, where $R_a = \prod_{i,j} G(i,j,\theta_{ij})$ and $G(i,j,\theta_{ij})$ are the Givens rotation matrices.
This approach was proposed by \citet{Quessard2020} to parametrize the group representation of the performed actions.

\begin{figure}[ht]
    \vspace{-5pt}
    \centering
    \includegraphics[width=0.42\textwidth]{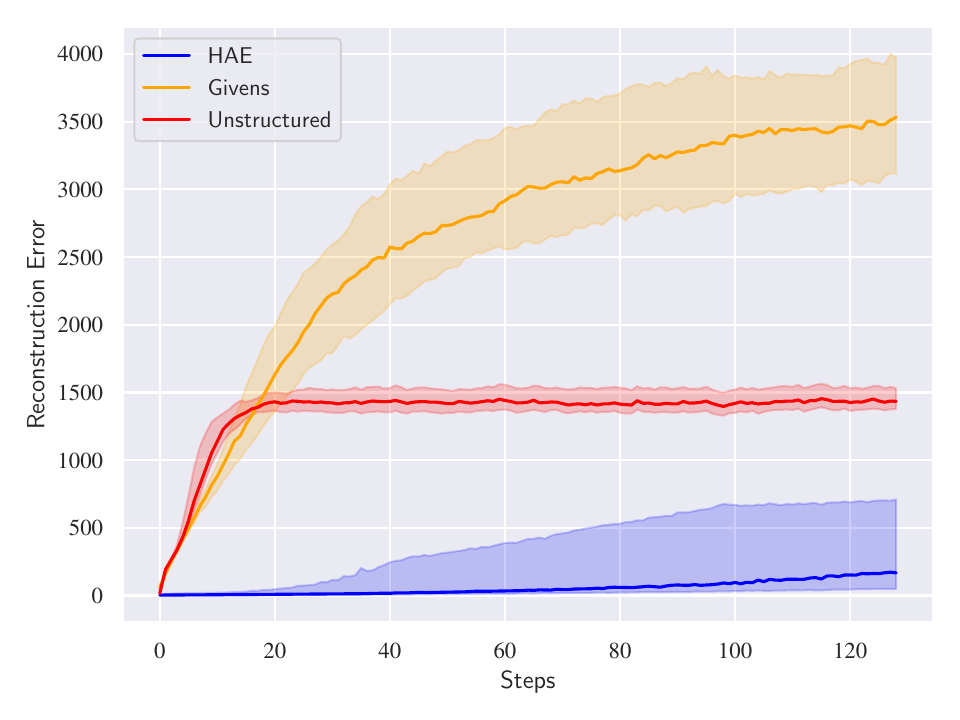}
    \vspace{-10pt}
    \caption{Step-wise reconstruction loss on the test dataset. Lines and shadings represent median and interquartile range over 50 random seeds.}
    \label{fig:recon_loss}
    \vspace{-10pt}
\end{figure}

We evaluate the methods in an offline setting, where we train each method on a given set of 2-step trajectories and test their generalization ability on a hold-out set of 128-step trajectories.
See Appendix~\ref{app:exp-rollout} for details on the setup.
Figure \ref{fig:recon_loss} shows the reconstruction loss for each method on the test trajectories.
Our result suggests that the HAE can outperform other methods significantly.

\subsection{Unsupervised Separation of Identity and Action}\label{sec:sampling}

\begin{figure}[ht]
    \vspace{0pt}
    \centering
\includegraphics[width=0.85\linewidth]{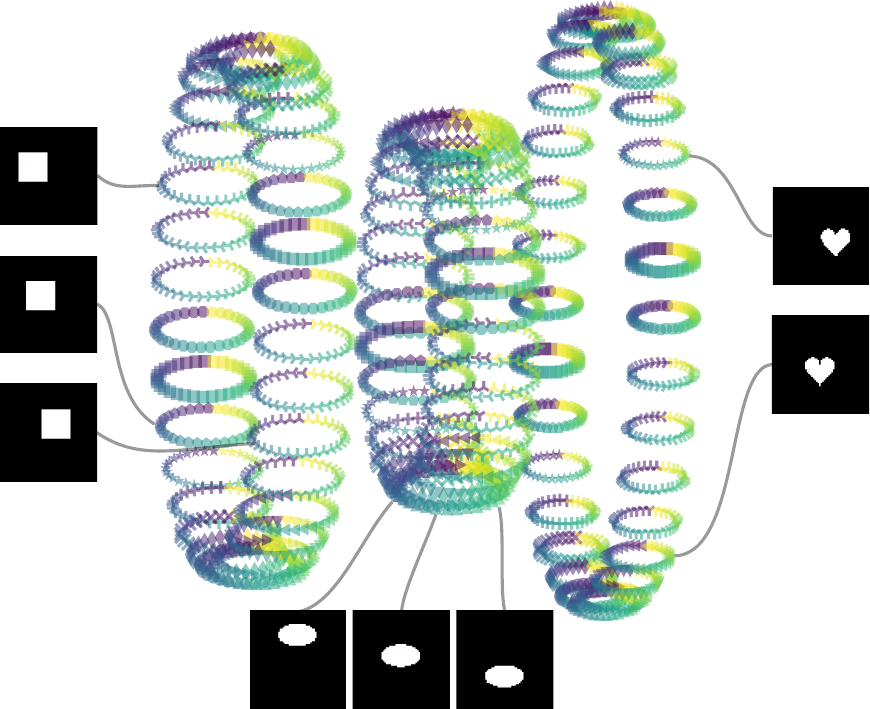}
    \caption{Projection of the $8-$dimensional representation vectors $z$ of the translated dsprites objects. $4$ units of $z$ encode position information, while one direction of the remaining $4$ dimensional space encodes shape identity. Colors indicate $y$ position, markers indicate $x$ position and each torus corresponds to a different shape.}
    \label{fig:muliobject-e}
    
\end{figure}

Next, we consider the subset of the dSprites dataset consisting of three shapes (heart, square and ellipse) acted on by the previous group of 2D cyclic translations $G = C_x \times C_y$. This action is not \textit{transitive} (see Appendix~\ref{app:grp}). 
Indeed, given a world state $w$ corresponding to a given shape at a given location on the torus, the set $G\cdot w$ of all world states we can transition to corresponds to the same shape at all possible locations. 
This is called the \textit{orbit} of $w$ (see Appendix~\ref{app:grp}).
We have 3 orbits, one for each shape.
We do not assume knowledge of any of this and train the HAE using the soft constraint for disentanglement described in Section~\ref{sec:soft-dis} using a representation space of dimension $D~=~8$.

Results show that the model learns a group representation $\rho$ (see Figure~\ref{fig:rho-multishape}) similar to the one learned in the single shape example,

where 4 dimensions correspond to the space acted on non trivially by $\rho$. 
 Figure \ref{fig:muliobject-e} confirms the HAE learns the representation of the cyclic translation group action shared among shapes, but also learns to separate the representations according to object shape along additional latent dimensions, giving rise to three distinct $G$-invariant tori. 
This is reminiscent of the two streams hypothesis of visual processing \citep{Goodale1992}, in which the "What" pathway processes information related to object identity, while the "Where" pathway processes information related to the object pose, relevant for manipulation. 
See Appendix~\ref{app:exp-multiobjects} for a more thorough analysis of this setup.

\subsection{Learning $SO(3)$ Structured Representations.}\label{sec:exp_bunny}
We now investigate whether HAEs can learn representations involving more complex groups than those decomposable into $SO(2)$ subgroups. We consider 2D images of a 3D bunny rotated in 3D space. The group acting on the states of this object is $SO(3)$. It is a 3D manifold with a rather complex topology (notably it is connected, but not simply connected) and contrary to $SO(2)$ it is not commutative. $SO(3)$ is also a $3-$dimensional simple group, it cannot be decomposed as a direct product of non-trivial subgroups. A $3D$ rotation can be described by three angles of consecutive $1D$ rotations around orthogonal axis, a popular choice is the Euler angles in the order Z-Y-X: yaw, pitch  and roll.

We train the HAE on a dataset of $2-$step transitions with small Euler angles $g$ sampled uniformly in the interval $[-0.5,0.5] rad$. The angles are passed through the fixed random neural network $\varphi$ before being forwarded through the group representation $\rho$. We continue to write $g$ for both $g$ and $\varphi(g)$. 
Figure~\ref{fig:bunny_recons} shows the quality of reconstructions on a test $128-$step transition.

\begin{figure}[t]
    \centering
    \includegraphics[width=0.8\linewidth]{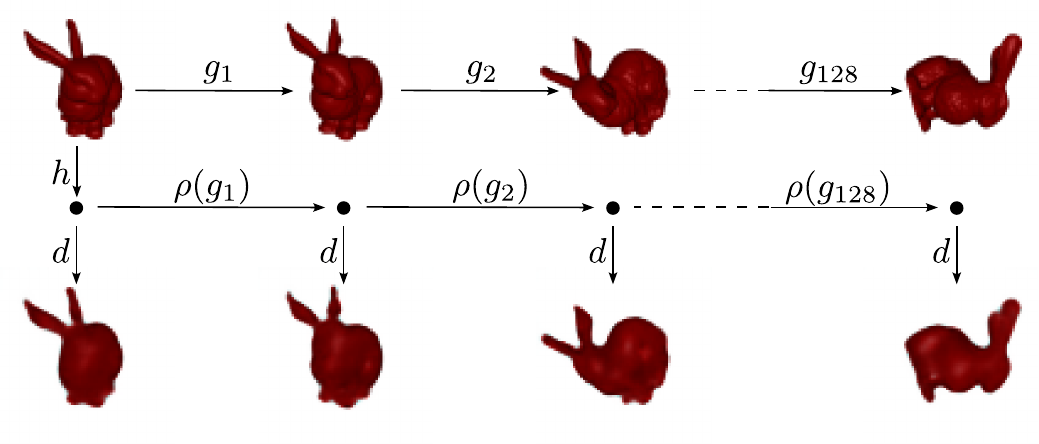}
    \vspace*{-.5cm}
    \caption{Example reconstructions of a 128-step interaction. Top and bottom row respectively correspond to the original and reconstructed observations.
    \label{fig:bunny_recons}}
    \vspace*{-.5cm}
\end{figure}

We choose a representation dimension large enough and let the sparsity loss (See Section~\ref{sec:soft-dis}) trivialize unnecessary dimensions. 
Figure~\ref{fig:rots_fixcol} shows the group representation $\rho(g)$ for specific rotations $g$.  
The group representation emerges block-diagonal due to the sparsity loss despite the group not being decomposable.
The blocks correspond to the irreducible representations of $SO(3)$.
One might expect a single $3$ dimensional block consisting of rotation matrices, which would be enough to act on the position of a point-like particle on the surface of a sphere, however it cannot correspond to a transitive action over $3D$ orientations because the orbits by the action would be spheres (a $2D$ manifold). 
Instead, we obtain two irreducible representations of dimensions $3$ and $5$.
This could be seen as the decomposition of a $9-$dimensional representation of $SO(3)$.
Which makes sense assuming each orientation of the bunny is encoded as a rotation matrix.
As such, the $9-$dimensional representation corresponds to the action of $SO(3)$ on $3\times3$ matrices through $R\cdot M = RMR^{T}$ with $R\in SO(3)$.

One of the benefits of the parametrization $\rho=\exp\circ \phi$ is to leverage the Lie algebra of the group of transitions. Indeed, as the mapping $\rho=exp\circ\phi$ converges to a group representation of $SO(3)$, 
$\phi(g)$ for different transitions $g$ spans the Lie algebra and a basis (group generators) can be obtained by applying PCA on the set of $\{\phi(g)\}_{g\in G}$ inferred from observed samples. 
Figure~\ref{fig:bunny_traversals_fixcols} shows that linear traversals of the Lie algebra correspond to $1$D rotations visualized through their action on the bunny's orientation. It is interesting to note we obtain full-circle rotations through these traversals even though the model was trained solely on transitions with rotations smaller than $30\degree$.

\begin{figure}[ht]
    
    \centering
    \includegraphics[width=\linewidth]{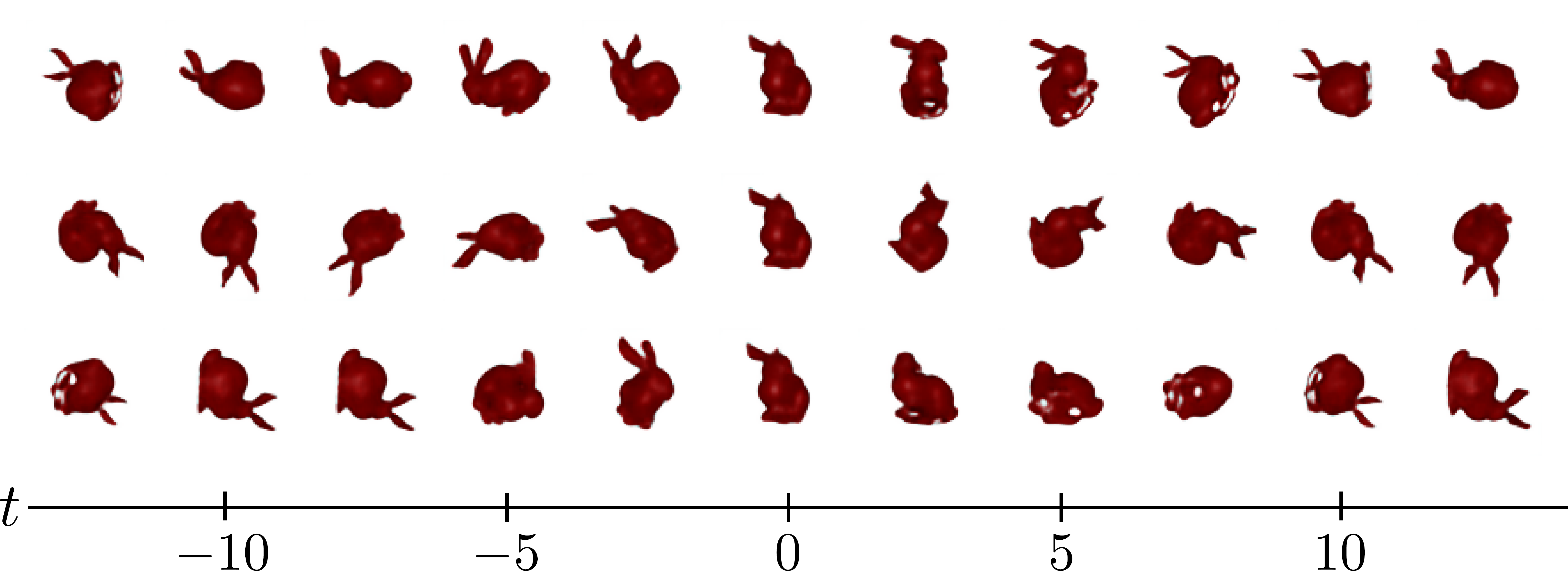}
    \vspace*{-.60cm}
    \caption{Linear traversal of the Lie algebra along the 3 principal components.}
    \label{fig:bunny_traversals_fixcols}
    \vspace*{-.10cm}
\end{figure}

\begin{figure}[ht]
    
    \centering
    \includegraphics[width=0.7\linewidth]{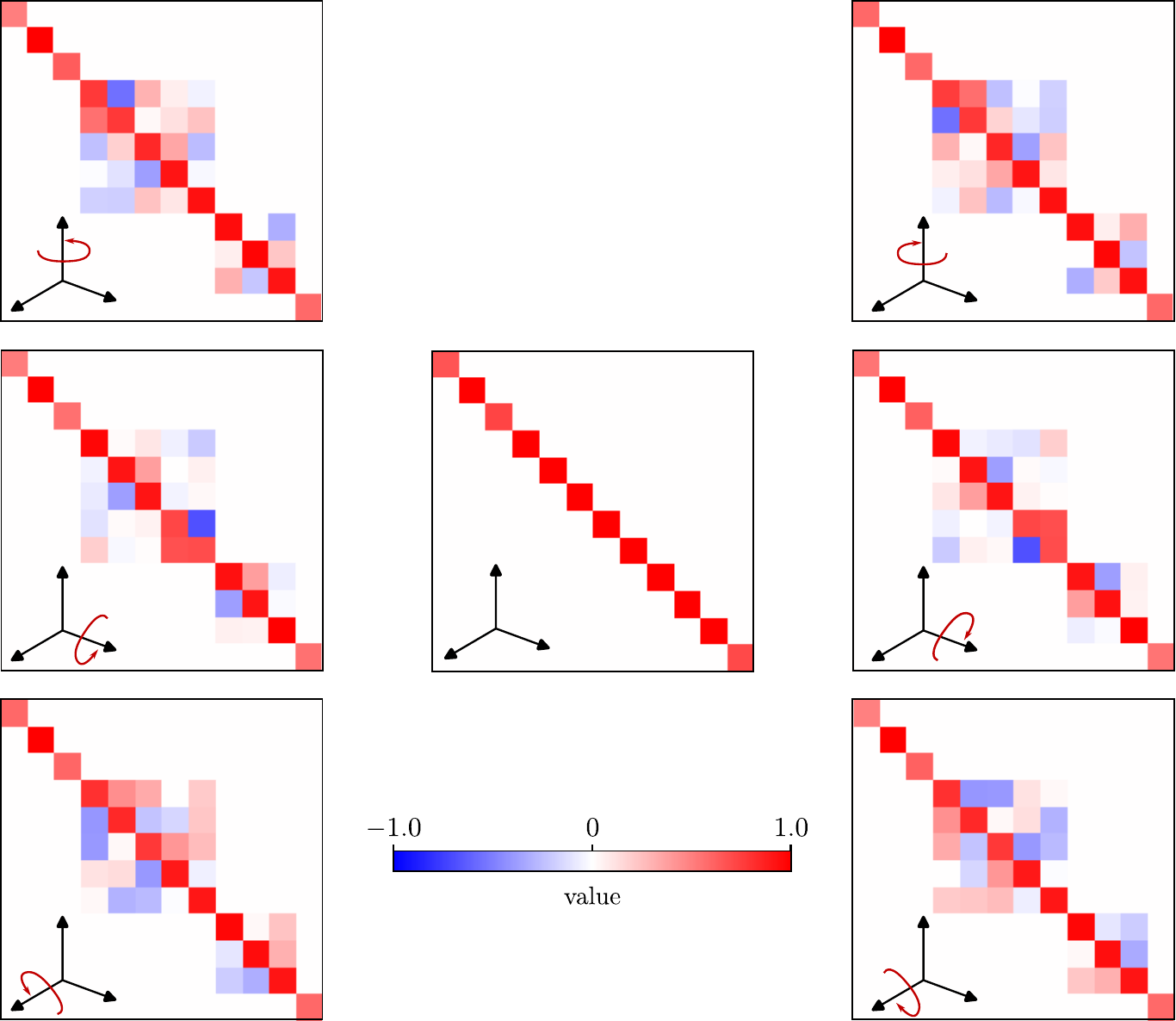}
    \caption{Group representation $\rho(g)$ for example actions $g$ corresponding to identity (center), and rotations $\pi/8$ (left) and $-\pi/8$ (right) around the yaw axis (top), the pitch axis (middle) and the roll axis (bottom).}
    \label{fig:rots_fixcol}
    \vspace*{-.25cm}
\end{figure}

To appreciate how well the group structure is captured, we evaluate the reconstruction error over long 128-step rollouts in Figure~\ref{fig:rollouts_fixcol}. We compare our approach to an unstructured forward model trained on the same dataset. 
Compared to the unstructured model, the HAE keeps a low reconstruction error since at every step the latent remains on the representation manifold. 

When training the HAE on a similar dataset where the action also shifts the bunny's hue, the resulting Lie algebra is decomposed into color and rotation as seen through the traversals in Figure~\ref{fig:bunny_traversals}.

More details are provided in the Appendix~\ref{app:exp-bunny}.

\begin{figure}[ht]
    
    \centering
    \includegraphics[width=0.8\linewidth]{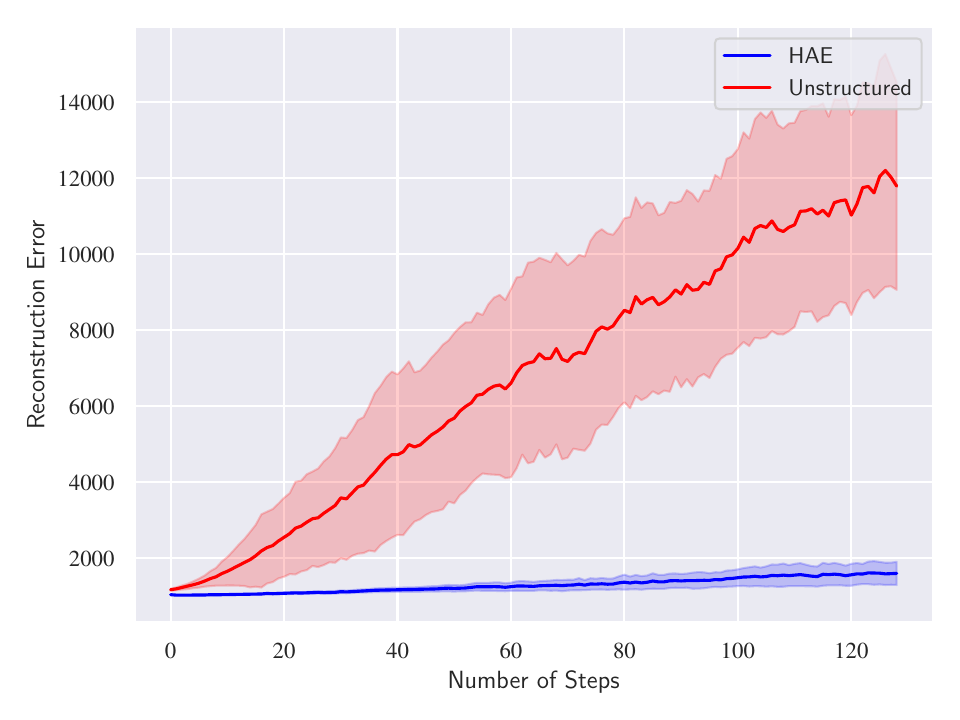}
    \vspace*{-.25cm}
    \caption{Reconstruction error for $128-$step rollouts. Line and shadings correspond respectively to the mean and standard deviation  over $20$ seeds.}
    \label{fig:rollouts_fixcol}
    \vspace*{-.25cm}
\end{figure}

\begin{figure}[ht]
    
    \centering
    \includegraphics[width=0.9\linewidth]{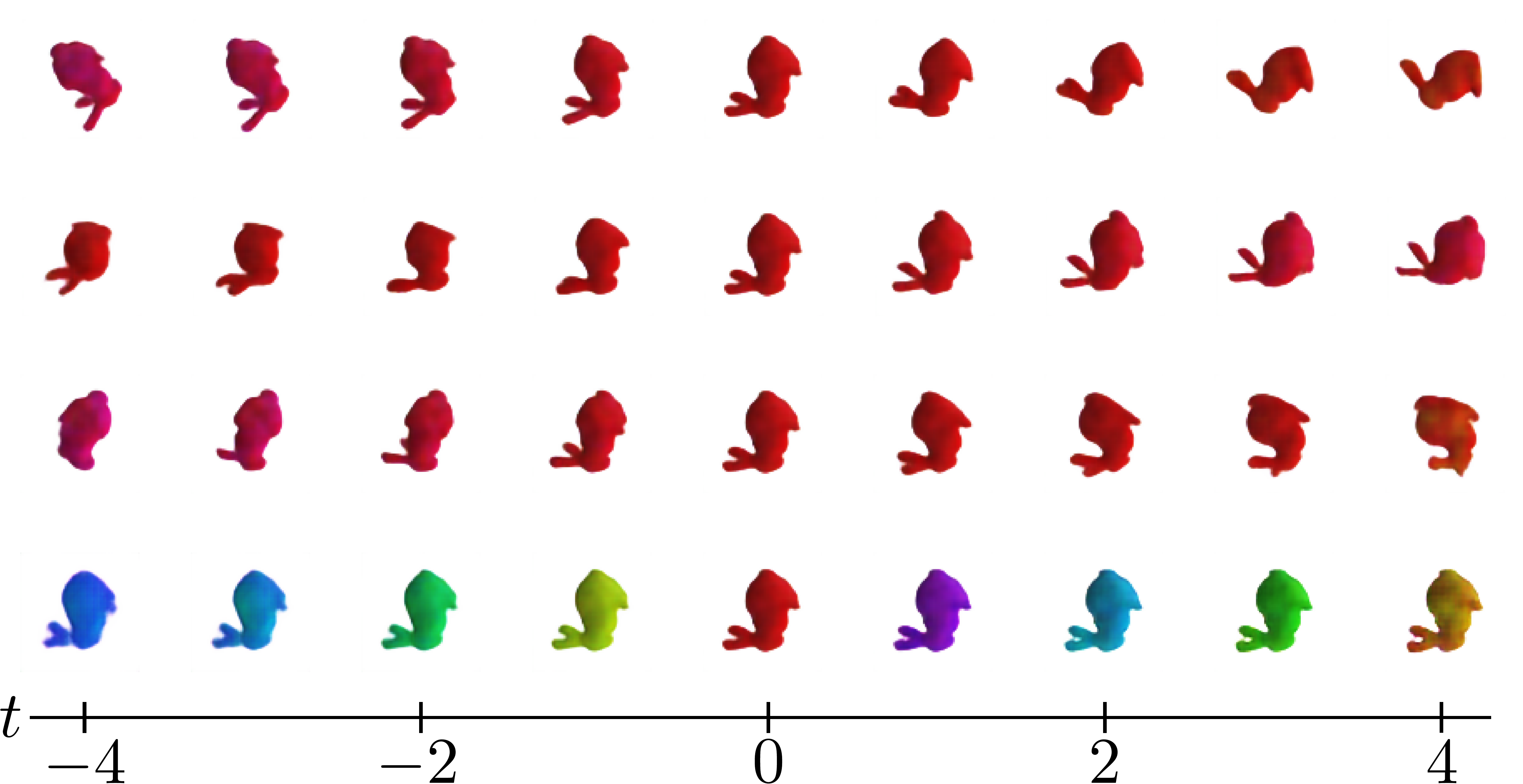}
    \vspace*{-.15cm}
    \caption{Linear traversal of the Lie algebra along the 4 principal components.}
    \label{fig:bunny_traversals}
    \vspace*{-.25cm}
\end{figure}

\vspace*{-.3cm}

\section{Discussion}\label{sec:discussion}
\vspace*{-.2cm}
We provide theoretical and experimental evidence supporting that the HAE allows an agent to
infer the geometric structure of its environment by learning a corresponding group-structured low-dimensional internal manifold. In contrast to earlier works, 

our assumptions are relaxed to general properties of the structuring group, involving in particular compactness and connectedness. 
Unlike other methods, we do not constrain observed actions to be performed along the generative factors. One limitation of our framework (as well as of previous works) 
is the deterministic nature of the mapping between state and environment. The intrinsically non-linear nature of the problem makes the theoretical analysis more challenging in the stochastic setting and is left to future work.

Our theoretical result emphasizes the benefits of a 2-step prediction objective to learn the group structure of the latent representation. In addition, the requirement of an additional loss on the reconstruction of the observations resonates with the central role played by predictive coding frameworks in neuroscientific accounts of representation learning \citep{rao1999predictive,pezzulo2022evolution}. In this sense, our result suggests an interplay between different learning objectives that, when achieved jointly, may endow human and artificial agents with interventional world representations.

When using a set of cyclic actions, we find 

that the HAE learns a toroidal manifold similar to the embeddings reported in neuroscience, for example, in the hippocampus of mice where toroid-like embeddings
encode the animals' head-direction \citep{Chaudhuri2019}.
When considering an environment consisting of different objects acted on by the same group of transformations, we find that the HAE learns the shared geometric structure and separates the objects into different orbits, resulting in a quotient space of transformation invariant representations of the objects. 
In particular, the emergence of these invariant object representations suggests that the enforced group structure allows the unsupervised learning of additional behavior-relevant information encoded in the latent representation. 

We further provided experimental support for the ability of HAEs to learn the action of more complex groups such as $SO(3)$ and showed that HAEs's Lie algebra structure of the learned group representation can be explored using PCA. 

In {\bf conclusion}, we have shown theoretically and experimentally that the proposed HAE approach constitutes a step towards moving from statistical representations to latent interventional representations, learned in a self-supervised manner without reinforcement signal. Subject to relatively limited assumptions, these representations successfully capture the environment's ground-truth interventional group structure.

\subsection*{Acknowledgements}
H.K. was supported by CLS. 
H-R.P. was supported by the International Max Planck Research School for Intelligent Systems. 
This work was supported by the Swiss National Science Foundation (B.F.G. CRSII5-173721 and
315230\_189251), ETH funding (B.F.G.\ ETH-20 19-01), the Human Frontiers Science Program
(RGY0072/2019) and the German Federal Ministry of Education and Research (BMBF): Tübingen AI Center, FKZ: 1IS18039B.

\clearpage
\bibliography{bibliography}

\begin{thebibliography}{47}
\providecommand{\natexlab}[1]{#1}
\providecommand{\url}[1]{\texttt{#1}}
\expandafter\ifx\csname urlstyle\endcsname\relax
  \providecommand{\doi}[1]{doi: #1}\else
  \providecommand{\doi}{doi: \begingroup \urlstyle{rm}\Url}\fi

\bibitem[Ansel et~al.(2024)Ansel, Yang, He, Gimelshein, Jain, Voznesensky, Bao,
  Bell, Berard, Burovski, Chauhan, Chourdia, Constable, Desmaison, DeVito,
  Ellison, Feng, Gong, Gschwind, Hirsh, Huang, Kalambarkar, Kirsch, Lazos,
  Lezcano, Liang, Liang, Lu, Luk, Maher, Pan, Puhrsch, Reso, Saroufim,
  Siraichi, Suk, Suo, Tillet, Wang, Wang, Wen, Zhang, Zhao, Zhou, Zou, Mathews,
  Chanan, Wu, and Chintala]{Ansel_PyTorch_2_Faster_2024}
Ansel, J., Yang, E., He, H., Gimelshein, N., Jain, A., Voznesensky, M., Bao,
  B., Bell, P., Berard, D., Burovski, E., Chauhan, G., Chourdia, A., Constable,
  W., Desmaison, A., DeVito, Z., Ellison, E., Feng, W., Gong, J., Gschwind, M.,
  Hirsh, B., Huang, S., Kalambarkar, K., Kirsch, L., Lazos, M., Lezcano, M.,
  Liang, Y., Liang, J., Lu, Y., Luk, C., Maher, B., Pan, Y., Puhrsch, C., Reso,
  M., Saroufim, M., Siraichi, M.~Y., Suk, H., Suo, M., Tillet, P., Wang, E.,
  Wang, X., Wen, W., Zhang, S., Zhao, X., Zhou, K., Zou, R., Mathews, A.,
  Chanan, G., Wu, P., and Chintala, S.
\newblock {PyTorch 2: Faster Machine Learning Through Dynamic Python Bytecode
  Transformation and Graph Compilation}.
\newblock In \emph{29th ACM International Conference on Architectural Support
  for Programming Languages and Operating Systems, Volume 2 (ASPLOS '24)}. ACM,
  April 2024.
\newblock \doi{10.1145/3620665.3640366}.
\newblock URL \url{https://pytorch.org/assets/pytorch2-2.pdf}.

\bibitem[Antonyan et~al.(2009)Antonyan, Antonyan, and
  Rodr{\'\i}guez-Medina]{antonyan2009linearization}
Antonyan, N., Antonyan, S.~A., and Rodr{\'\i}guez-Medina, L.
\newblock Linearization of proper group actions.
\newblock \emph{Topology and its Applications}, 156\penalty0 (11):\penalty0
  1946--1956, 2009.

\bibitem[Bach et~al.(2011)Bach, Jenatton, Mairal, and Obozinski]{Bach2011}
Bach, F., Jenatton, R., Mairal, J., and Obozinski, G.
\newblock Optimization with sparsity-inducing penalties.
\newblock \emph{Foundations and Trends in Machine Learning}, 4:\penalty0 --,
  2011.
\newblock ISSN 19358237.
\newblock \doi{10.1561/2200000015}.

\bibitem[Bader et~al.(2019)Bader, Blanes, and Casas]{Bader2019}
Bader, P., Blanes, S., and Casas, F.
\newblock Computing the matrix exponential with an optimized taylor polynomial
  approximation.
\newblock \emph{Mathematics 2019, Vol. 7, Page 1174}, 7:\penalty0 1174, 12
  2019.
\newblock ISSN 2227-7390.
\newblock \doi{10.3390/MATH7121174}.

\bibitem[Bengio et~al.(2013)Bengio, Courville, and Vincent]{Bengio2012}
Bengio, Y., Courville, A., and Vincent, P.
\newblock Representation learning: A review and new perspectives.
\newblock \emph{IEEE Trans. Pattern Anal. Mach. Intell.}, 35\penalty0
  (8):\penalty0 1798–1828, aug 2013.
\newblock ISSN 0162-8828.
\newblock \doi{10.1109/TPAMI.2013.50}.
\newblock URL \url{https://doi.org/10.1109/TPAMI.2013.50}.

\bibitem[Besserve et~al.(2018)Besserve, Shajarisales, Sch{\"o}lkopf, and
  Janzing]{besserve2018aistats}
Besserve, M., Shajarisales, N., Sch{\"o}lkopf, B., and Janzing, D.
\newblock Group invariance principles for causal generative models.
\newblock In \emph{AISTATS}, 2018.

\bibitem[Besserve et~al.(2021)Besserve, Sun, Janzing, and
  Sch\"olkopf]{Besserve_Sun_Janzing_Scholkopf_2021}
Besserve, M., Sun, R., Janzing, D., and Sch\"olkopf, B.
\newblock A theory of independent mechanisms for extrapolation in generative
  models.
\newblock In \emph{Proceedings of the AAAI Conference on Artificial
  Intelligence}, pp.\  6741--6749, 2021.

\bibitem[Caselles-Dupr{\'{e}} et~al.(2019)Caselles-Dupr{\'{e}}, Garcia-Ortiz,
  and Filliat]{Caselles-Dupre2019}
Caselles-Dupr{\'{e}}, H., Garcia-Ortiz, M., and Filliat, D.
\newblock {Symmetry-based disentangled representation learning requires
  interaction with environments}.
\newblock In \emph{Advances in Neural Information Processing Systems},
  volume~32, 2019.

\bibitem[Chaudhuri et~al.(2019)Chaudhuri, Gerçek, Pandey, Peyrache, and
  Fiete]{Chaudhuri2019}
Chaudhuri, R., Gerçek, B., Pandey, B., Peyrache, A., and Fiete, I.
\newblock The intrinsic attractor manifold and population dynamics of a
  canonical cognitive circuit across waking and sleep.
\newblock \emph{Nature Neuroscience}, 22, 2019.

\bibitem[Cohen \& Welling(2016)Cohen and Welling]{CohenW16a}
Cohen, T.~S. and Welling, M.
\newblock {Steerable CNNs}.
\newblock \emph{CoRR}, abs/1612.0, 2016.

\bibitem[Connor \& Rozell(2020)Connor and Rozell]{Connor2020}
Connor, M. and Rozell, C.
\newblock Representing closed transformation paths in encoded network latent
  space.
\newblock \emph{Proceedings of the AAAI Conference on Artificial Intelligence},
  34\penalty0 (04):\penalty0 3666--3675, Apr. 2020.
\newblock \doi{10.1609/aaai.v34i04.5775}.
\newblock URL \url{https://ojs.aaai.org/index.php/AAAI/article/view/5775}.

\bibitem[Dehmamy et~al.(2021)Dehmamy, Walters, Liu, Wang, and Yu]{Dehmamy2021}
Dehmamy, N., Walters, R., Liu, Y., Wang, D., and Yu, R.
\newblock Automatic symmetry discovery with lie algebra convolutional network.
\newblock In \emph{Advances in Neural Information Processing Systems},
  volume~4, pp.\  2503--2515, 9 2021.
\newblock ISBN 9781713845393.

\bibitem[Dodwell(1983)]{Dodwell1983}
Dodwell, P.~C.
\newblock The {Lie} transformation group model of visual perception.
\newblock \emph{Perception \& Psychophysics}, 34, 1983.
\newblock ISSN 00315117.
\newblock \doi{10.3758/BF03205890}.

\bibitem[Eslami et~al.(2018)Eslami, Rezende, Besse, Viola, Morcos, Garnelo,
  Ruderman, Rusu, Danihelka, Gregor, Reichert, Buesing, Weber, Vinyals,
  Rosenbaum, Rabinowitz, King, Hillier, Botvinick, Wierstra, Kavukcuoglu, and
  Hassabis]{AliEslami2018}
Eslami, S. M.~A., Rezende, D.~J., Besse, F., Viola, F., Morcos, A.~S., Garnelo,
  M., Ruderman, A., Rusu, A.~A., Danihelka, I., Gregor, K., Reichert, D.~P.,
  Buesing, L., Weber, T., Vinyals, O., Rosenbaum, D., Rabinowitz, N., King, H.,
  Hillier, C., Botvinick, M., Wierstra, D., Kavukcuoglu, K., and Hassabis, D.
\newblock Neural scene representation and rendering.
\newblock \emph{Science}, 360:\penalty0 1204--1210, 2018.
\newblock ISSN 10959203.
\newblock \doi{10.1126/science.aar6170}.

\bibitem[Finzi et~al.(2021)Finzi, Welling, and Wilson]{Finzi2021}
Finzi, M., Welling, M., and Wilson, A.~G.
\newblock A practical method for constructing equivariant multilayer
  perceptrons for arbitrary matrix groups.
\newblock \emph{38th International Conference on Machine Learning}, 2021.

\bibitem[Goodale \& Milner(1992)Goodale and Milner]{Goodale1992}
Goodale, M.~A. and Milner, A.~D.
\newblock Separate visual pathways for perception and action.
\newblock \emph{Trends in Neurosciences}, 15, 1992.
\newblock ISSN 01662236.
\newblock \doi{10.1016/0166-2236(92)90344-8}.

\bibitem[Gresele et~al.(2021)Gresele, von Kügelgen, Stimper, Schölkopf, and
  Besserve]{gresele_independent_2021}
Gresele, L., von Kügelgen, J., Stimper, V., Schölkopf, B., and Besserve, M.
\newblock Independent mechanism analysis, a new concept?
\newblock In \emph{Advances in Neural Information Processing Systems}, 2021.

\bibitem[Ha \& Schmidhuber(2018)Ha and Schmidhuber]{ha2018world}
Ha, D. and Schmidhuber, J.
\newblock World models.
\newblock \emph{arXiv preprint arXiv:1803.10122}, 2018.

\bibitem[Hall(2015)]{hall2015}
Hall, B.
\newblock \emph{Lie groups, Lie algebras, and representations}.
\newblock Springer, 2015.

\bibitem[Higgins et~al.(2018)Higgins, Amos, Pfau, Racaniere, Matthey, Rezende,
  and Lerchner]{Higgins2018}
Higgins, I., Amos, D., Pfau, D., Racaniere, S., Matthey, L., Rezende, D., and
  Lerchner, A.
\newblock Towards a definition of disentangled representations.
\newblock \emph{arXiv preprint arXiv:1812.02230}, 2018.

\bibitem[Hyvarinen \& Morioka(2016)Hyvarinen and
  Morioka]{hyvarinen_unsupervised_2016}
Hyvarinen, A. and Morioka, H.
\newblock Unsupervised feature extraction by time-contrastive learning and
  nonlinear {ICA}.
\newblock \emph{Advances in neural information processing systems}, 29, 2016.

\bibitem[Keller et~al.(2012)Keller, Bonhoeffer, and Hübener]{Keller2012}
Keller, G., Bonhoeffer, T., and Hübener, M.
\newblock Sensorimotor mismatch signals in primary visual cortex of the
  behaving mouse.
\newblock \emph{Neuron}, 74, 2012.

\bibitem[Khemakhem et~al.(2020)Khemakhem, Kingma, Monti, and
  Hyvarinen]{khemakhem_variational_2020}
Khemakhem, I., Kingma, D., Monti, R., and Hyvarinen, A.
\newblock Variational {Autoencoders} and {Nonlinear} {ICA}: {A} {Unifying}
  {Framework}.
\newblock In \emph{International {Conference} on {Artificial} {Intelligence}
  and {Statistics}}, pp.\  2207--2217. PMLR, June 2020.
\newblock ISSN: 2640-3498.

\bibitem[Kondor \& Trivedi(2018)Kondor and Trivedi]{Kondor2018}
Kondor, R. and Trivedi, S.
\newblock On the generalization of equivariance and convolution in neural
  networks to the action of compact groups.
\newblock \emph{35th International Conference on Machine Learning}, 6:\penalty0
  4295--4309, 2 2018.

\bibitem[Kraft \& Russell(2014)Kraft and Russell]{kraft2014families}
Kraft, H. and Russell, P.
\newblock Families of group actions, generic isotriviality, and linearization.
\newblock \emph{Transformation groups}, 19\penalty0 (3):\penalty0 779--792,
  2014.

\bibitem[Kulkarni et~al.(2015)Kulkarni, Whitney, Kohli, and
  Tenenbaum]{kulkarni2015deep}
Kulkarni, T.~D., Whitney, W.~F., Kohli, P., and Tenenbaum, J.
\newblock Deep convolutional inverse graphics network.
\newblock \emph{Advances in neural information processing systems}, 28, 2015.

\bibitem[LeCun(1989)]{LeCun1989}
LeCun, Y.
\newblock Generalization and network design strategies.
\newblock In Pfeifer, R., Schreter, Z., Fogelman, F., and Steels, L. (eds.),
  \emph{Connectionism in Perspective}, Zurich, Switzerland, 1989. Elsevier.
\newblock an extended version was published as a technical report of the
  University of Toronto.

\bibitem[Lee(2013)]{lee2013smooth}
Lee, J.~M.
\newblock Smooth manifolds.
\newblock In \emph{Introduction to Smooth Manifolds}, pp.\  1--31. Springer,
  2013.

\bibitem[Locatello et~al.(2019)Locatello, Bauer, Lucie, Rätsch, Gelly,
  Schölkopf, and Bachem]{Locatello2019}
Locatello, F., Bauer, S., Lucie, M., Rätsch, G., Gelly, S., Schölkopf, B.,
  and Bachem, O.
\newblock Challenging common assumptions in the unsupervised learning of
  disentangled representations.
\newblock \emph{36th International Conference on Machine Learning, ICML 2019},
  2019-June:\penalty0 7247--7283, 2019.

\bibitem[Matthey et~al.(2017)Matthey, Higgins, Hassabis, and
  Lerchner]{dsprites17}
Matthey, L., Higgins, I., Hassabis, D., and Lerchner, A.
\newblock dsprites: Disentanglement testing sprites dataset.
\newblock https://github.com/deepmind/dsprites-dataset/, 2017.

\bibitem[Painter et~al.(2020)Painter, Hare, and Prügel-Bennett]{Painter20}
Painter, M., Hare, J., and Prügel-Bennett, A.
\newblock Linear disentangled representations and unsupervised action
  estimation.
\newblock \emph{Advances in Neural Information Processing Systems}, 33, 2020.

\bibitem[Park et~al.(2022)Park, Biza, Zhao, Van De~Meent, and
  Walters]{park2022learning}
Park, J.~Y., Biza, O., Zhao, L., Van De~Meent, J.-W., and Walters, R.
\newblock Learning symmetric embeddings for equivariant world models.
\newblock In \emph{International Conference on Machine Learning}, pp.\
  17372--17389. PMLR, 2022.

\bibitem[Pezzulo et~al.(2022)Pezzulo, Parr, and Friston]{pezzulo2022evolution}
Pezzulo, G., Parr, T., and Friston, K.
\newblock The evolution of brain architectures for predictive coding and active
  inference.
\newblock \emph{Philosophical Transactions of the Royal Society B},
  377\penalty0 (1844):\penalty0 20200531, 2022.

\bibitem[Piaget(1964)]{Piaget1964a}
Piaget, J.
\newblock {Part I: Cognitive development in children: Piaget development and
  learning}.
\newblock \emph{Journal of Research in Science Teaching}, 2\penalty0
  (3):\penalty0 176--186, 1964.
\newblock ISSN 10982736.
\newblock \doi{10.1002/tea.3660020306}.

\bibitem[Quessard et~al.(2020)Quessard, Barrett, and Clements]{Quessard2020}
Quessard, R., Barrett, T.~D., and Clements, W.~R.
\newblock {Learning disentangled representations and group structure of
  dynamical environments}.
\newblock In \emph{Advances in Neural Information Processing Systems}, 2020.

\bibitem[Rao \& Ballard(1999)Rao and Ballard]{rao1999predictive}
Rao, R.~P. and Ballard, D.~H.
\newblock Predictive coding in the visual cortex: a functional interpretation
  of some extra-classical receptive-field effects.
\newblock \emph{Nature neuroscience}, 2\penalty0 (1):\penalty0 79--87, 1999.

\bibitem[Sch{\"o}lkopf et~al.(2021)Sch{\"o}lkopf, Locatello, Bauer, Ke,
  Kalchbrenner, Goyal, and Bengio]{Scholkopfetal21}
Sch{\"o}lkopf, B., Locatello, F., Bauer, S., Ke, N.~R., Kalchbrenner, N.,
  Goyal, A., and Bengio, Y.
\newblock Toward causal representation learning.
\newblock \emph{Proceedings of the IEEE}, 109\penalty0 (5):\penalty0 612--634,
  2021.

\bibitem[Schrittwieser et~al.(2020)Schrittwieser, Antonoglou, Hubert, Simonyan,
  Sifre, Schmitt, Guez, Lockhart, Hassabis, Graepel,
  et~al.]{schrittwieser2020mastering}
Schrittwieser, J., Antonoglou, I., Hubert, T., Simonyan, K., Sifre, L.,
  Schmitt, S., Guez, A., Lockhart, E., Hassabis, D., Graepel, T., et~al.
\newblock Mastering atari, go, chess and shogi by planning with a learned
  model.
\newblock \emph{Nature}, 588\penalty0 (7839):\penalty0 604--609, 2020.

\bibitem[Sontakke et~al.(2021)Sontakke, Mehrjou, Itti, and
  Sch{\"o}lkopf]{sontakke2021causal}
Sontakke, S.~A., Mehrjou, A., Itti, L., and Sch{\"o}lkopf, B.
\newblock Causal curiosity: {RL} agents discovering self-supervised experiments
  for causal representation learning.
\newblock In \emph{International Conference on Machine Learning}, pp.\
  9848--9858. PMLR, 2021.

\bibitem[Sutton \& Barto(2015)Sutton and Barto]{Sutton1998}
Sutton, R.~S. and Barto, A.~G.
\newblock \emph{Reinforcement Learning: An Introduction}.
\newblock MIT Press, Cambridge, MA, 2nd edition, 2015.

\bibitem[Thomas et~al.(2017)Thomas, Pondard, Bengio, Sarfati, Beaudoin, Meurs,
  Pineau, Precup, and Bengio]{Thomas2017}
Thomas, V., Pondard, J., Bengio, E., Sarfati, M., Beaudoin, P., Meurs, M.-J.,
  Pineau, J., Precup, D., and Bengio, Y.
\newblock Independently controllable factors, 2017.
\newblock URL \url{https://arxiv.org/abs/1708.01289}.

\bibitem[Tonnaer et~al.(2022)Tonnaer, Rey, Menkovski, Holenderski, and
  Portegies]{Tonnaer2020}
Tonnaer, L., Rey, L. A.~P., Menkovski, V., Holenderski, M., and Portegies, J.
\newblock Quantifying and learning linear symmetry-based disentanglement.
\newblock In \emph{International Conference on Machine Learning}, pp.\
  21584--21608. PMLR, 2022.

\bibitem[Turk \& Levoy(1994)Turk and Levoy]{turk1994zippered}
Turk, G. and Levoy, M.
\newblock Zippered polygon meshes from range images.
\newblock In \emph{Proceedings of the 21st annual conference on Computer
  graphics and interactive techniques}, pp.\  311--318, 1994.

\bibitem[van~der Pol et~al.(2020)van~der Pol, Worrall, van Hoof, Oliehoek, and
  Welling]{vanderpol20b}
van~der Pol, E., Worrall, D., van Hoof, H., Oliehoek, F., and Welling, M.
\newblock Mdp homomorphic networks: Group symmetries in reinforcement learning.
\newblock In Larochelle, H., Ranzato, M., Hadsell, R., Balcan, M., and Lin, H.
  (eds.), \emph{Advances in Neural Information Processing Systems}, volume~33,
  pp.\  4199--4210. Curran Associates, Inc., 2020.
\newblock URL
  \url{https://proceedings.neurips.cc/paper_files/paper/2020/file/2be5f9c2e3620eb73c2972d7552b6cb5-Paper.pdf}.

\bibitem[von Kügelgen et~al.(2021)von Kügelgen, Sharma, Gresele, Brendel,
  Schölkopf, Besserve, and Locatello]{von_kugelgen_self-supervised_2021}
von Kügelgen, J., Sharma, Y., Gresele, L., Brendel, W., Schölkopf, B.,
  Besserve, M., and Locatello, F.
\newblock Self-{Supervised} {Learning} with {Data} {Augmentations} {Provably}
  {Isolates} {Content} from {Style}.
\newblock In \emph{Advances in Neural Information Processing Systems}, 2021.

\bibitem[Yang et~al.(2021)Yang, Ren, Wang, Zeng, and Zheng]{Tao2022}
Yang, T., Ren, X., Wang, Y., Zeng, W., and Zheng, N.
\newblock Towards building a group-based unsupervised representation
  disentanglement framework, 2021.

\bibitem[Zhou et~al.(2021)Zhou, Knowles, and Finn]{Zhou2020}
Zhou, A., Knowles, T., and Finn, C.
\newblock Meta-learning symmetries by reparameterization.
\newblock In \emph{International Conference on Learning Representations}, 2021.

\end{thebibliography}
\bibliographystyle{icml2023}

\newpage
\appendix

\onecolumn

\section{Background on group theory}\label{app:grp}

In this section, we provide an overview of group theory concepts exploited in this work.

\begin{definition}[Group]\label{def:grp}
A set $G$ is a group if it is equipped with a binary operation $\cdot:G\times G\rightarrow G$ and if the group axioms are satisfied
\begin{enumerate}
    \item Associativity: $\forall a,b,c\in G$, $(a\cdot b)\cdot c = a\cdot(b\cdot c)$
    \item Identity: There exists $e\in G$ such that $\forall a \in G$, $a\cdot e = e \cdot a = a$.
    \item Inverse: $\forall a\in G$, there exists $b\in G$ such that $a\cdot b = b\cdot a = e$. This inverse is denoted $a^{-1}$.
\end{enumerate}
\end{definition}
We are often interested in sets of transformations, which respect a group structure, but are applied to objects that are not necessarily group elements. 
This can be studied through group actions, which describe how groups \emph{act} on other mathematical entities. 
\begin{definition}[Group Action]\label{def:grp_act}
Given a group $G$ and a set $X$, a group action is a function $\cdot_X:G\times~X~\rightarrow~X$ such that the following conditions are satisfied.
\begin{enumerate}
    \item Identity: If $e\in G$ is the identity element, then $e\cdot_X x = x$, $\forall x \in X$.
    \item Compatibility: $\forall g,h \in G$ and $\forall x \in X, \quad g\cdot_X (h\cdot_X x) = (g\cdot h)\cdot_X x$
\end{enumerate}
\end{definition}

The group action $\cdot_X: G\times~X~\rightarrow~X$ induces a group homomorphism $\rho_{\cdot_X}:G\rightarrow~Sym(X)$.
(where $Sym(X)$ is the group of all invertible transformations of $X$)
through:
\[
\forall (g,x)\in G\times X,\quad \rho_{\cdot_X}(g)(x):= g \cdot_X x 
\]
The group homomorphism property of $\rho_{\cdot_X}$ comes from the group action axioms of $\cdot_X$:
\[
\rho_{\cdot_X}(id)(x) = id \cdot_X x = x\quad \textrm{(identity)}
\]
\[
= id_X(x)
\]
So $\rho_{\cdot_X}(id) = id_X$.
and
\[
\rho_{\cdot_X}(g_1\cdot g_2)(x) = (g_1\cdot g_2) \cdot_X x = g_1 \cdot_X (g_2 \cdot_X x) \quad \textrm{(compatibility)}
\]
\[
=\rho_{\cdot_X}(g_1) \circ \rho_{\cdot_X}(g_2) (x)
\]
Equality over all of $X$ leads to equality of the functions: $\rho_{\cdot_X}(g_1\cdot~g_2)~=\rho_{\cdot_X}(g_1)\circ~\rho_{\cdot_X}(g_2)$. 

In what follows, we are interested in \emph{linear group actions} in which case the acted on space is a vector space $V$ and the induced homomorphism $\rho$ maps $G$ to the group $GL(V)$ of invertible linear transformations of $V$.
This mapping is called a group representation. 
Actions of this type have been studied extensively in representation theory.

\begin{definition}[Group Representation]\label{def:grp_rep}
Let $G$ be a group and $V$ a vector space. A representation is a function $\rho:G\rightarrow GL(V)$ such that $\forall g, h\in G$, one has $\rho(g)\rho(h) = \rho(g\cdot h)$.
\end{definition}
Note that such definition is not restricted to finite dimensional vector spaces, however we will limit our study to this case, such that representations are appropriately described by mappings from $G$ to a space of square matrices. 
 
\begin{definition}[Lie Group]
A Lie Group $G$ is a nonempty set satisfying the following conditions:
\begin{itemize}
    \item $G$ is a group.
    \item $G$ is a smooth manifold.
    \item The group operation $\cdot:G\times G \rightarrow G$ and the inverse map $.^{-1}:G\rightarrow G$ are smooth.
\end{itemize}
\end{definition}

We limit ourselves to the study of linear Lie Groups, Lie groups that are matrix groups.
The tangent space to a Lie Group at the identity forms a Lie Algebra. A Lie Algebra $\mathfrak{g}$ is a vector space equipped with a bilinear map $[.,.]:\mathfrak{g}\times \mathfrak{g} \rightarrow \mathfrak{g}$ called the Lie Bracket.
We will not introduce the Lie Bracket as we do not make use of it.
The Lie Algebra somehow describes most of everything happening in its Lie Group. This connection is established through the exponential map.
\begin{definition}[Exponential Map]
The exponential map $exp: \mathfrak{g} \rightarrow G$ is defined for matrix Lie Groups by the series:
\[
e^A = \sum_{k=0}^\infty \frac{1}{k} A^k.\quad \forall A \in \mathfrak{g}
\]
\end{definition}
The exponential map is not always surjective. However if we only consider groups that are connected and compact, the exponential is surjective, which justifies our parametrization of the group representation through:
\[
    \rho: G \xrightarrow{\phi} \mathfrak{g}=M_n(\mathbb{R}) \xrightarrow{exp} GL_n(\mathbb{R})
\]
Where $\phi$ is a trainable arbitrary mapping.

\paragraph{Lie Groups, Algebras and Representations}\label{app:rhophi}
When a group $G$ is also a smooth manifold it is called a Lie Group.
The tangent space to the group $G$ at the identity forms a Lie Algebra $\mathfrak{g}$: a vector space equipped with a bilinear ``Lie bracket'' operator $[\,.\,,\,.\,]$.
 Importantly, when the Lie group has a representation $\rho:G\to GL(V)$, 
the Lie algebra can also be associated to a (matrix) representation $R:\mathfrak{g}\to \mathfrak{gl}(V)$, which is the differential of $\rho$ and where  $\mathfrak{gl}(V)$ denotes the Lie algebra consisting of all linear endomorphisms of $V$, with bracket given by $[ X , Y ] = X Y - Y X$. 
Restricting ourselves to matrix Lie groups, we will leverage a convenient link between these two representations through the \textit{matrix exponential} given by the series $\exp(A) = \sum_{k=0}^\infty \frac{1}{k!} A^k$ for arbitrary square matrix $A$. When $\rho$ is faithful (i.e. injective) and under certain assumptions on the group -- for instance if $G$ is a connected, compact matrix Lie group \citep[Corollary 11.10]{hall2015} 
the so-called exponential map of the group writes
\[
\begin{matrix}
   \exp_G :& \mathfrak{g}&\to & G\\
    & A & \mapsto &\rho^{-1}\circ\exp ( R(A))
\end{matrix}
\]
and is surjective. 
Therefore, the whole group $G$ can be described using elements of the Lie algebra. As a consequence, ${\exp}_G$ has a right inverse and for any group element $g$ we can write ${\exp}_G\circ {\exp}_G^{-1}(g)=g$, which leads to
\[
\exp \circ R \circ {\exp}_G^{-1}(g)=\rho(g)\,.
\]
Thus $\rho(g)$ can be decomposed as $\exp \circ \, \phi(g)$ with $\phi = R \circ {\exp}_G^{-1}$.
From an algorithmic perspective, we can thus learn $\rho$ by exploiting differentiable implementations of the matrix exponential \citep{Bader2019} and fit the mapping $\phi$ with a neural network and backpropagation. Because for different values $g\in G$, the elements $\phi(g)$ live in a vector space, they may be easier to learn than their  group representation $\rho(g)$ which lives on a manifold. For illustration of this aspect, the standard group and algebra representations for $SO(2)$ are, respectively,
\[
\rho(\theta)=\left[\begin{matrix}
    \cos(\theta)&\sin(\theta)\\-\sin(\theta)&\cos(\theta)
\end{matrix} \right]\quad\mbox{ and }\quad
R(\theta)=\theta \left[\begin{matrix}
    0&1\\-1&0
\end{matrix} \right]\,.
\]
The Lie algebra structure can be leveraged even if we do not consider a set of transformations covering the whole Lie group. 
For example, discrete image translations of an integer number of pixels along one dimension can be approximated (by enforcing periodicity at image boundaries) by a cyclic subgroup of the 2D rotation group $SO(2)$, where the rotation angle is constrained to be a multiple $2\pi/n$, with $n$ being the number of pixels of the image along this dimension.

\paragraph{Group action types}
The effect of a group action on a base space $X$ varies according to the properties of the homomorphism defined by the group action 
\begin{align*}
    \tau:&G \rightarrow Sym(X) \\
          &g \mapsto g \cdot_X \square
\end{align*}
We introduce two types of actions:
\begin{definition}[Transitive Group Action]
        The action of $G$ on $X$ is \emph{transitive} if $X$ forms a single \emph{orbit}. 
        
        in other words, $\forall x,y\in X,\exists g\in G; g\cdot x=y$.
\end{definition}        

\begin{definition}[Faithful Group Action]
        The action of $G$ on $X$ is \emph{faithful} if the homomorphism $G\rightarrow Sym(X)$ corresponding to the action is bijective (an isomorphism). 
        
        In that case, $\forall g_1\neq g_2 \in G, \exists x \in X;\quad g_1\cdot x \neq g_2 \cdot x$ .
\end{definition}

We also define the \emph{orbits} by a group action:
\begin{definition}[Orbit by a Group Action]
        The orbit of an element $x\in X$ by the action $\cdot_X$ of a group $G$ is the set
        \[
        G\cdot_X x = \{g\cdot_X x : g\in G\}
        \]
\end{definition}
When the action of $G$ is transitive on $X$, then $X$ is the single orbit by the action of $G$:
\[
\forall x \in X, G\cdot_X x = X
\]
Such is the case for our experiments using a single shape. We also explore the case where the action is not transitive in the multi shape experiment visualized in Figure~\ref{fig:muliobject-e}.

\section{Theoretical foundations of HAEs}\label{app:proofs}

\subsection{Proofs}
The following propositions are established for more general losses than the ones defined in main text, which have been simplified for the sake of readability. In this setting, we consider sequences of a possibly large number $n$ of actions and we sum $N$-step losses for all possible subsequences and not only the sequences at $t=1$ (which is the setting of main text). This leads to the losses
\begin{equation}
    \mathcal{L}_{pred}^N (\rho,h)\!=\!\! \sum_t\sum_{j=1}^{N}  
\left\|h(o_{t+j}) \!-\! \Big(\prod_{i=0}^{j-1}\rho(g_{t+i})\Big)h(o_t)\right\|_2^2 \,,
\end{equation}
and
\begin{equation}
\mathcal{L}_{rec}^N(\rho,h,d)\! = \!\! \sum_t\sum_{j=0}^N 
\left\|o_{t+j} \!-\! d\left( \!\Big(\prod_{i=0}^{j-1}\rho(g_{t+i})\Big)h(o_t)\!\right)\right\|_2^2\,.
\end{equation}
Where the empty product $\prod_{i=0}^{-1}\rho(g_{t+i})$ is the identity matrix by convention. 

Additionally, given that we are considering settings where the observed data distribution is restricted to a submanifold of the ambient space (e.g. the pixel space), we introduce formal tools to model continuous random variables in such setting (we leave aside the easier case of a discrete state space). We restrict ourselves to smooth manifold structures, and introduce the notion of smooth positive manifold density, which essentially generalises the Lebesgue measure to manifolds for integration theory, and in particular, associates a strictly positive value to each open neighborhood of the manifold. We refer to \citep[Chapter 16]{lee2013smooth} for a detailed introduction of the concept. Such density always exist in such case (but is not unique) \citep[Proposition 16.37]{lee2013smooth}.
Let $\mu$ be such a positive manifold density, then one can define a continuous random variable $w$ that have a \textit{compactly supported density} $p_w$ with respect to $\mu$ on $M$ in the following sense: there exist a continuous positive function $p_w: M\to \mathbb{R}$ with compact support, such that for any set $A$ in the Borel $\sigma$ algebra (the smallest containing all open sets), we have:
\[
P(w\in A) = \int_M p_w \mu\,.
\]
When applying a diffeomorphism to such variable, the result will also have a density that can be computed using the classical change of variable formula.

First we provide a formal statement of the main theoretical result of the paper together with its proof, and then proceed to prove the other propositions.

\begin{prop}\label{prop:twosteps}
    Under generative model of Section~\ref{sec:sbdrl} with $b$ diffeomorphic onto its image and Assumption~\ref{assum:linearworld}, with a smooth positive manifold density $\mu$ on $W$. Consider a setting where the distribution of observations $(o_1,g_1,...,o_n)$ is generated according to Section~\ref{sec:sbdrl}, such that there exists $t\in \{1,\dots,n-1\}$ for which the corresponding world state-actions joint distribution $(w_t,g_t,g_{t+1})$ has a density $p$ with respect to $\mu\times \mu_G\times \mu_G$, where $\mu_H$ is the Haar probability measure of $G$ (i.e. the unique probability measure invariant by the action of any group element), and such that the support of $p$ takes the form $\mbox{supp}(p)=S_i\times G\times G$ where $S_i$ is a $G$ invariant distribution of initial states. Let $\gamma>0$, if $(\rho,h,d)$ are continuous and minimize the expectation of $\mathcal{L}_{pred}^2(\rho,h) + \gamma \mathcal{L}_{rec}^k(\rho,h,d)$, for $k\geq 0$, then $\rho$ can be restricted to the linear group $GL(V)$ of a vector subspace $V$ such that $\rho_{V}:G \to GL(V)$ is a non-trivial group representation and $(\rho_{V},h)$ is a symmetry-based representation. 
\end{prop}


\begin{proof}
Under Assumption~\ref{assum:linearworld}, the true state space $W$ is acted upon by $G$ through its representation $(\rho^*,W^*)$. As such, $m \circ b^{-1}$, the inverse of the generating process composed with the injective mapping $m$ of the assumption, and $\rho^*$ verify $\mathcal{L}_{pred}^2(\rho^*,m\circ b^{-1}) =0$ and $\mathcal{L}_{rec}^k(\rho^*,m\circ b^{-1}, b\circ m^{-1}) =0$. 

As a consequence, if $(\rho,h,d)$ minimizes the full loss, it is a zero of this loss, which implies it is a zero of both the reconstruction and prediction losses.

This further implies that $(\rho,h,d)$ minimizes the 0-step reconstruction loss, which entails that $h$ is injective by Proposition~\ref{prop:inject}.

We assumed $(\rho, h)$ minimizes $\mathbb{E}_{(o_1,g_1,...,o_n)}[\mathcal{L}_{pred}^2(\rho,h)]$ therefore  $\mathbb{E}_{(o_1,g_1,...,o_n)}[\mathcal{L}_{pred}^2(\rho,h)]=0$.

Observed transitions $(o_t, g_t, o_{t+1})$ correspond to an action on the true world states $w_{t+1} = g_t \cdot_W w_t$ --- $(\rho, h)$. 
Under our assumptions, by contradiction, any 2-step sequence in the support of the distribution verifies:
\begin{equation}\label{eq:onestepeq}
\rho(g_t) h(o_t) = h(o_{t+1})
\end{equation}
and
\begin{equation}\label{eq:twostepeq}
\rho(g_{t+1})\rho(g_{t}) h(o_{t}) = h(o_{t+2})\,.    
\end{equation}
Indeed, assuming otherwise, there would exist a value $(w_t, g_t, w_{t+1}, g_{t+1}, w_{t+2})$ where this is not satisfied. By continuity of $h$, $\rho$ and the norm, there would exist a neighborhood $U$ such that for all $(w_t', g_t', w_{t+1}', g_{t+1}', w_{t+2}')\in U$ either
\[
\|\rho(g_t) h(o_t) - h(o_{t+1})\|^2>0
\]
or
\[
\|\rho(g_{t+1}')\rho(g_{t}') h(o_{t}') - h(o_{t+2}')\|^2>0\,.
\]
and since the distribution has strictly positive density, integrating norms over the neighborhood would lead to a strictly positive loss and thus to a contradiction.

Let us prove $\rho$ is a group representation, meaning it verifies $\rho(g_2 g_1) =\rho(g_2) \rho(g_1), \forall g_1, g_2\in G$.

Let $g_1, g_2, g_3 \in G$ such that $ g_3 = g_2 g_1$.

Let $w_1$ in the support of $\mu_i$, then by assumption $w_2, w_3$ which verify $w_2 = g_1\cdot_W w_1$, $w_3 = g_2\cdot_W w_2= g_3\cdot_W w_1$ also belong to the support of $\mu_i$ and we have the following commutative diagram for the associated observed transitions. 
\begin{figure}[h!]
    \centering
    \includegraphics[width=0.4\textwidth]{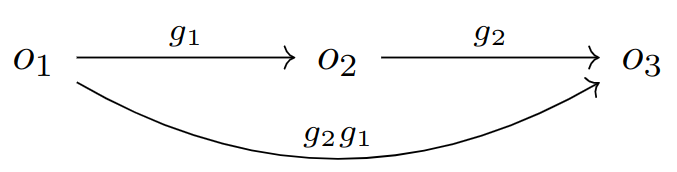}
    \label{fig:1step2steps}
\end{figure}

Consider then any 2-step transitions 
$(o_t,g_t,o_{t+1},g_{t+1},o_{t+2})$ and $(o_{t}',g_{t}',o_{t+1}',g_{t+1}',o_{t+2}')$ such that:
\[
\begin{cases}
o_t = o_{t}' = o_1\\ 
o_{t+1} = o_2\\
o_{t+2} = o_{t+1}' = o_3\\
\end{cases}
\]
and
\[
\begin{cases}
g_t = g_1\\
g_{t+1} = g_2\\
g_{t}' = g_3 = g_2g_1\,.\\
\end{cases}
\]
Those belong to the support of the sample path distribution by assumption, such that,
using (\ref{eq:twostepeq}), 
$\mathcal{L}_{pred}^2 = 0$ implies:
$\rho(g_{t+1})\rho(g_t)h(o_t) = h(o_{t+2})$ therefore $\rho(g_2)\rho(g_1)h(o_1) = h(o_3)$.
Moreover, using (\ref{eq:onestepeq}) 
$\rho(g_{t}')h(o_{t}') = h(o_{t+1}')$ therefore $\rho(g_3)h(o_1) = h(o_3)$.

Therefore we have $\rho(g_2)\rho(g_1)h(o_1) = \rho(g_3)h(o_1)$ or $\rho(g_2)\rho(g_1)h(o_1) = \rho(g_2 g_1)h(o_1)$

Let us define $V= \mbox{span} (h\circ b(\mbox{supp}(\mu_i)))$. As the above equality is verified for any $o_1$ in the image of the support of $\mu_{i}$ by $b$ 
we get by linearity that $\rho(g_2)\rho(g_1) z = \rho(g_2 g_1)z$ for all $z\in V$ (equality of linear mappings over vectors that span a vector subspace). 

Moreover, $V$ is stable by $\rho$. Indeed, 
consider one element $z\in V$ and $g\in G$, then
$z=\sum_k h\circ b(w_k)$, with all $w_k \in \mbox{supp}(\mu_i)$. As a consequence, they satisfy (using (\ref{eq:onestepeq}))
\[
\rho(g) h\circ b(w_{k}) = h\circ b(g \cdot_W w_{k})\in h\circ b(\mbox{supp}(\mu_i))\,,
\]
for all $k$, by assumed $G$-stability of $\mbox{supp}(\mu_i)$. As a consequence 
\[
\rho(g) z = \sum_k \rho(g) h\circ b(w_{k}) \in \mbox{span}(h\circ b(\mbox{supp}(\mu_i)))=V\,,
\]
Therefore the restriction of $\rho_V$ of $\rho$ to $V$ satisfies $\rho_{V}(g_1 g_2)=\rho_{V}(g_1)\rho_{V} (g_2)$ and is a group representation of $G$. 

The injectivity assumption ensures $h(O)$ does not collapse to a single element.

Let us show $h$ is a group-structured representation.

By (\ref{eq:onestepeq}), we have for every observed transition $(w_t, g_t)$ (i.e. $w_t$ taking any value in the support of the density $\mu_i$ and $g_t$ taking any value in $G$)
\[
h(o_{t+1}) = \rho(g_{t})h(o_t)
\]
where $o_t=b(w_t)$ and $o_{t+1}=b(w_{t+1})$, by the generative model assumptions.
Then 
\[
o_{t+1} = b(w_{t+1}) = b(g_t \cdot_W w_t)
\]
such that
\[
h \circ b (g_t \cdot_W w_t) = \rho(g_t)h \circ b(w_t)\,.
\]

\end{proof}
\paragraph{Remarks.} We provide some justification for the assumption on the distribution of the observed sample paths for steps $t,..., t+2$ exploited in the formal version of Proposition~\ref{prop:twosteps}. Taking a $G$-invariant support $S_i$ associated to the distribution of states at time $t$ ensures that the same states are visited at $t$ and subsequent steps. This can be typically achieved by choosing a time $t$ such that the transition form initial states $w_1$ have reached all points of the orbits of this state such those are associated with a positive density. Conditions to achieve this with specific choice of transition probability can be studied with the framework of Markov chains on general spaces, and relate to the concept of ``mixing''. The restriction to have $W$ with a manifold structure addresses the case of uncountable infinite world states, while maintaining enough structure to be able to define probability densities. Implicitly, the restricts the result to those manifold where the Lebesgue measure can be defined, which covers the broad case of oriented Riemannian manifolds. Another version of our result is easily obtained in the simpler case of a countable set of world states. 
For simplicity, we also assume a positive density for transitions $g_t$ and $g_{t+1}$ with respect to the Haar measure. This allows us to prove our result by only focusing on 2-step transitions. However, this may be considered a strong assumption when comparing to practical settings were group elements for each transitions may not be selected from the whole group. For example, for connected Lie groups, it may be more natural to consider transition picked from a neighborhood of the identity (e.g. rotation with a maximal angle), which limit the amount of changes from one state to the next. We believe it is possible to generalize our result to such setting by exploiting a broader range of time points of the observed sequences, although this would require more technical assumptions and proof developments. Finally, the resulting restriction of the representation to subspace $V$ essentially reflects the boundaries of the explored state space, restricted to the support of $\mu_i$.

We now provide some more additional theoretical results related to the HAE structure and learning objective.

\begin{restatable}{prop}{gmap}\label{prop:gmap}
Under generative model of Section~\ref{sec:sbdrl} with $b$ diffeomorphic onto its image and Assumption~\ref{assum:linearworld}, we assume additionally that $W$ has a smooth manifold structure with smooth positive manifold density $\mu$. Consider a setting where the distribution of observations $(o_1,g_1,...,o_n)$ is generated according to Section~\ref{sec:sbdrl}, such that there exists $t\in \{1,\dots,n-1\}$ for which the corresponding world state-actions joint distribution $(w_t,g_t,g_{t+1})$ has a density $p$ with respect to $\mu\times \mu_G\times \mu_G$, where $\mu_H$ is the Haar probability measure of $G$ (i.e. the unique probability measure invariant by the action of any group element), and such that the support of $p$ takes the form $\mbox{supp}(p)=S_i\times G\times G$ where $S_i$ is a $G$ invariant distribution of initial states.
Assume we have access to the ground truth representation $\rho^*$. If $h$ minimizes the expectation of $\mathcal{L}_{pred}^1(\rho^*,h)$ then $h$ is a symmetry-based representation, meaning $h\circ b$ is equivariant.
\end{restatable}
\begin{proof}[Proof]
Under Assumption~\ref{assum:linearworld}, the true state space $W$ is acted upon by $G$ through its representation $(\rho^*,W^*)$. As such the inverse of the generating process composed with the injective mapping $m$ of the assumption $m \circ b^{-1}$ and $\rho^*$ verify $\mathcal{L}_{pred}^1(\rho^*,m\circ b^{-1}) =0$.


As a consequence, $h$ also achieves zero expected loss such that, following the same reasoning as the one leading to (\ref{eq:onestepeq}) in the proof of Proposition~\ref{prop:twosteps}
for all $o_t=b(w_t)$ with $w_t\in \mbox{supp}(\mu_i)$
\[
h(o_{t+1}) = \rho^*(g_{t})h(o_t)\,.
\]
by the generative model assumptions
\[
o_{t+1} = b(w_{t+1}) = b(\rho^*(g_t) w_t)\,.
\]
Also $o_t = b(w_t)$ such that
\[
h \circ b (\rho^*(g_t) w_t) = \rho^*(g_t)h \circ b(w_t)\,.
\]
\end{proof}
Proposition~\ref{prop:gmap} is clearly weaker than Proposition~\ref{prop:twosteps} notably in the sense that it requires assuming the knowledge of $\rho^*$. This is due to the limitation of the 1-step prediction loss. 

The next proposition further illustrates that  non-injectivity of $h$ may yield unstructured (trivial) representations when learning based on the prediction loss.

\begin{prop}\label{prop:trivial}
The trivial group representation $\rho = I$ (that always maps to the identity matrix) combined with a constant $h$ is a zero of the prediction loss $\mathcal{L}_{pred}^1 (\rho,h)$.
\end{prop}
\begin{proof}[Proof (sketch)]
    Choosing constant mappings $h(o)=C$ and $\rho_{g}=1$ for all $(o,g)$ trivially leads to a vanishing prediction loss.
\end{proof}

The next proposition shows that this injectivity can be enforced using the reconstruction loss.

\begin{restatable}{prop}{inject}\label{prop:inject}
Consider the setting of Proposition 1. 
Assume the encoder $h$ and the decoder $d$
are continuous and minimize $\mathcal{L}_{rec}^0 (\rho,h,d)$ then $h$ is injective. 
\end{restatable}

\begin{proof}
Assume the encoder $h$ and the decoder $d$  are continuous and minimize the 0-step reconstruction. Then the loss must be zero as choosing the left inverse $b^{-1}$ as encoder (and $b$ as decoder) achieves such loss value. 
\[
\forall o, o' \in O, \text{such that } o \neq o'.
\]
Then by continuity of $d$ and $h$
\[
d(h(o)) = o \text{ and } d(h(o')) = o'\,,
\]
as otherwise the expected loss would be stricly positive (similarly as the reasoning in Proposition~1).
Therefore
\[
d(h(o))\neq d(h(o')),\,
\]
which implies
\[
h(o) \neq h(o')\,.
\]

Therefore $h$ is injective.

\end{proof}

\subsection{Disentanglement through Sparsity Loss}\label{app:softblock}
To drop the assumption of prior knowledge of the number of subgroups in the representation as well as the dimension for their representation, we propose a sparsity loss (Equation \ref{eq:softblock}) on the group representation that prefers block-diagonal patterns (Example in Figure~\ref{fig:softblock}) by minimizing the terms of $(\rho_{ij}(g_t))_{ij}$ that do not belong to allowed non-zero terms. The loss is inspired from works on sparsity inducing losses \citep{Bach2011}.

\begin{equation*}\label{eq:softblocksupp}
    \mathcal{L}_{sparse}(\rho) = \sum_t \sum_{i\geq 0} \sqrt{ \sum_{j\geq i+1,\,k\leq i}\rho_{kj}(g_t)^2 + \rho_{jk}(g_t)^2 }\,.
\end{equation*}

The model is then trained on the composite loss:
\begin{equation*}
    \mathcal{L} = \mathcal{L}_{rec} + \gamma * \mathcal{L}_{pred} + 
    \delta *
    \mathcal{L}_{sparse}
\end{equation*}

We can then choose a representation space $Z=\RR^d$ with $d$ large enough to accommodate the total dimension of the group representation, extra dimensions are then trivialized by $\mathcal{L}_{sparse}$.

Figure \ref{fig:softblock} shows how the term associated with each value of $i$ induces a direct sum of 2 subrepresentations $\rho = \rho_1 \oplus \rho_2$ with a different split of the total representation dimension. Combining patterns produces any possible decomposition into a direct sum of subrepresentations.

\begin{figure}
    \centering
    \includegraphics[width=0.7\textwidth]{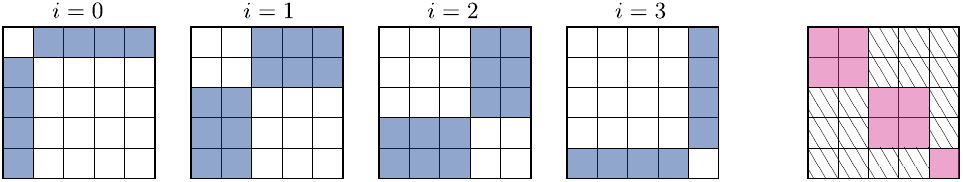}
    \caption{Left: Non allowed non-zero regions induced by each value of the index $i$ on a representation of dimension $5$. Right: an example pattern resulting from the activations of the terms corresponding to $i\in\{1,3\}$.}
    \label{fig:softblock}
\end{figure}

\section{Experiment: Learning representations structured by $G=SO(2)\times ... \times SO(2)$}\label{app:exp-single-shape}

\subsection{Setup}\label{app:setup}
We consider a subset of the dsprites dataset \citep{dsprites17} where a fixed scale and orientation heart is acted on by the group of 2D cyclic translations $G = C_x \times C_y$. $G$ is a discrete subgroup of $SO(2)\times SO(2)$.
The corresponding transition dataset contains tuples $(o_1,g_1,o_2,...,g_{N-1},o_N)$, where the observations $o_i$ are $64\times 64$ pixels and the transitions are given by $\varphi(g)$ where ${g = (g^x,g^y)\in G}$ represents the angular displacement and $\varphi$ is a two layer multi-layer perceptron, radomly initialized and fixed throughout the experiment. $\varphi(g)$ is a $50$-dimensional vector of non-linear mixture of $g^x$ and $g^y$. Details for $\varphi$ are summarized in Table~\ref{tab:varphi}.

\begin{table}[h!]
    \centering
    \caption{$\varphi$ architecture.}
    \begin{tabular}{cc}
        Parameter & Value \\ \midrule
        Linear Layers & [200, 50] \\
        Activation & ReLU \\
        Random Seed & 10\\
    \end{tabular}
    \label{tab:varphi}
\end{table}

\subsection{Learning a disentangled representation}
\subsubsection{Hyperparameters}\label{app:hp}
\paragraph{Model architecture}
We use a symmetrical architecture for the encoder and decoder, which we summarize in Table \ref{tab:network}.
The network was trained on the combined loss:
\[
\mathcal{L} = \mathcal{L}_{rec}^2(\rho,h,d) + \gamma \mathcal{L}_{pred}^2(\rho,h) + 
\delta
\mathcal{L}_{sparse}(\rho)
\]
Where we use the Binary Cross Entropy loss for the reconstruction term instead of the Mean Squared Error as it is better behaved during training.
\begin{table}[h!]
    \centering
    \caption{Network architecture.}
    \begin{tabular}{cc}
        Parameter & Value \\ \midrule
        Conv. Channels & [32, 32, 32, 32] \\
        Kernel Sizes & [6, 4, 4, 4] \\
        Strides & [2, 2, 1, 1] \\
        Linear Layer Size & 1024 \\
        Activation & ReLU \\
        Latent space & 4\\
        $\gamma$ & 400\\
        $\rho$ dimension & 8\\
        $\rho$ Linear Layers & [1024]\\
        $\rho$ activation & ReLU\\
        $\delta$ & $0.1$
    \end{tabular}
    \label{tab:network}
\end{table}

\paragraph{Training hyperparameters}
We trained the network using the hyperparameters summarized in Table~\ref{tab:hp}.
\begin{table}[h!]
    \centering
    \caption{Training hyperparameters.}
    \begin{tabular}{cc}
        Parameter & Value \\ \midrule
        Optimizer & Adam \\
        Learning rate & 0.001 \\
        Number of training sequences & 10000 \\
        Batch size & 500 \\
        Epochs & 101 \\
    \end{tabular}
    \label{tab:hp}
\end{table}

\subsubsection{Visualization}
We obtained manifold Figures~\ref{fig:manifold}~and~\ref{fig:muliobject-e}  by projecting the $D$-dimensional representation vectors of all images in the dataset on a random $2$D plane through Random Matrix Projection. We chose the projection with the most explainable visualization. 

\subsubsection{Representation of a Neighbourhood of Identity}
Figure~\ref{fig:moreactions} displays $\rho$ over a wider neighbourhood of the identity. 

\begin{figure}
    \centering
    \includegraphics[width=0.6\linewidth]{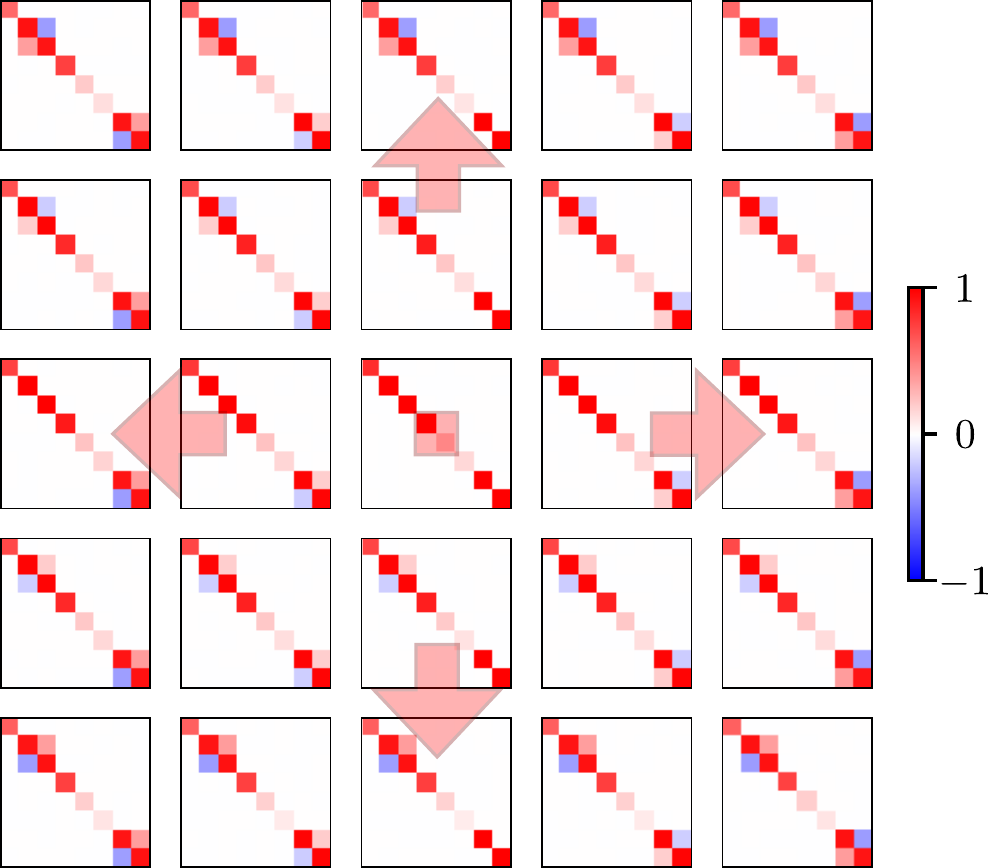}
    \caption{The representations of actions (horizontal and vertical displacements) in a square $\llbracket -2,2\rrbracket\times \llbracket-2,2\rrbracket$ around the identity for the dSprites experiment show a homomorphism structure.}
    \label{fig:moreactions}
\end{figure}

\subsubsection{Latent Traversal}\label{app:exp-alg}

We show how learning a mapping to the group algebra can be leveraged to navigate the group and the data manifold. We remind that $\rho = \exp\circ \phi$, where $\phi$ maps to the algebra $\mathfrak{g}$ of the group $G$, and $\exp$ is the matrix exponential which gives a connection between the algebra and the lie group.

The mapping $\phi$ and the group representation $\rho$ are constrained to be block diagonal matrices either through the strong constraint as seen in Section~\ref{sec:dis-grp-repr} or the soft constraint which uses a sparsity loss on $\rho$ described in Section~\ref{sec:soft-dis}. 
We consider the soft constraints case where no assumptions are made about the size of the matrices.
And we use the model trained as in Section~\ref{sec:exp_torus}.
Our matrices $\phi(g), \rho(g)$ are of dimension $8$.
However, we know that we overparametrized our group representation so we use PCA over the Algebra representation of a batch of transition signals $\phi(g)$ and find two principal components $A_1$ and $A_2$ that are matrices 
equal to zero except for two 2D skew symmetric blocks.

We obtain the figure~\ref{fig:traversal} by linearly traversing the algebra along its base vectors through ${tA_1}$ and  ${tA_2}$ for equally spaced values of $t\in \llbracket 0,10 \rrbracket $ and passing it to the matrix exponential which yields invertible matrices of the form ${R_{i,t} = e^{tA_i}}$. 
We encode an arbitrary initial observation to obtain its representation vector $z$, and traverse the latent space through $R_{i,t} z$. 
We decode the obtained vectors to obtain the predicted images.

The group algebra offers a smooth parametrization of the group and consequently of the data manifold and enables the prediction of observations evolution in the absence of performed actions. 
Indeed, in the above example, all transformations can be obtained in the form ${\exp(t_1 A_1 + t_2 A_2)}$ for $t_1, t_2 \in \RR$.
In addition, this parametrization allows us to obtain actions from the whole group, beyond the discretization considered and the limitation of the range of actions.


\subsubsection{Sparsity Based Disentanglement}\label{app:exp-softblock}
We show in Figure~\ref{fig:softblock} the learned group representation for the same dataset used in the experiment described in Appendix~\ref{app:setup} where instead of enforcing a block diagonal structure, it emerges under the influence of the sparsity inducing loss described in Appendix~\ref{app:softblock}. We do not assume knowledge of the representation dimension and we set it to a value high enough ($8$ in the example) and we find the HAE learns the disentangled representation found in the main experiment and trivializes the extra dimensions.  

\subsubsection{Additional Experiment}\label{app:add-exp}
We consider a subset of the dataset consisting of all variations of the heart under a fixed scale. As such the heart is acted on by the group $G=G_{\theta}\times G_x \times G_y = C_{39}\times C_{32} \times C_{32}$. We train a similar model to the one described in the section~\ref{app:hp} by changing the latent space to $6$ dimensions, the group representation is fixed of the form $\rho = \rho_1 \oplus \rho_2 \oplus \rho_3$ each of dimension $2$. We expect a disentangled representation space $Z=Z_1\oplus Z_2\oplus Z_3$. For the visualization of the learned representation manifold Figure~\ref{fig:xytheta}, we visualize each subspace $Z_i$ separately by only varying one generative factor and keeping all else fixed. We also visualized the learned representations for a subset of transitions corresponding to the elementary generative transitions for each subgroup in Figure~\ref{fig:repr_xytheta}.

\begin{figure}[ht]
    \centering
    \includegraphics[width=\textwidth]{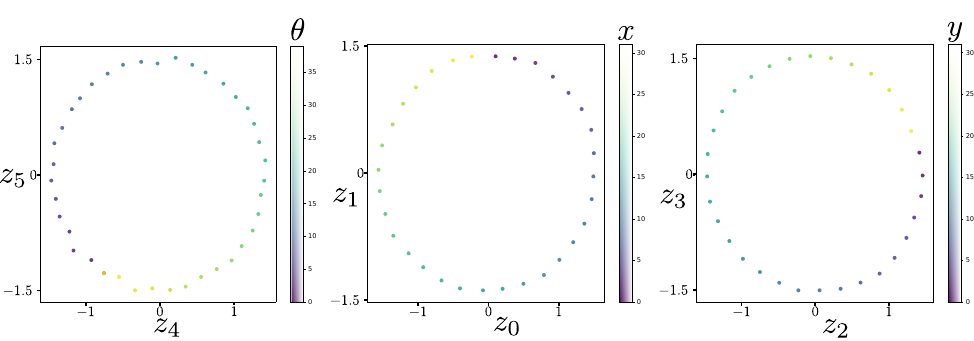}
    \caption{Visualization of the $6$D embedding vectors for the heart dataset. For each visualized $2$D subspace, we only vary the latent represented by the subspace.}
    \label{fig:xytheta}
\end{figure}

\begin{figure}[ht]
    \centering
    \includegraphics[width=\textwidth]{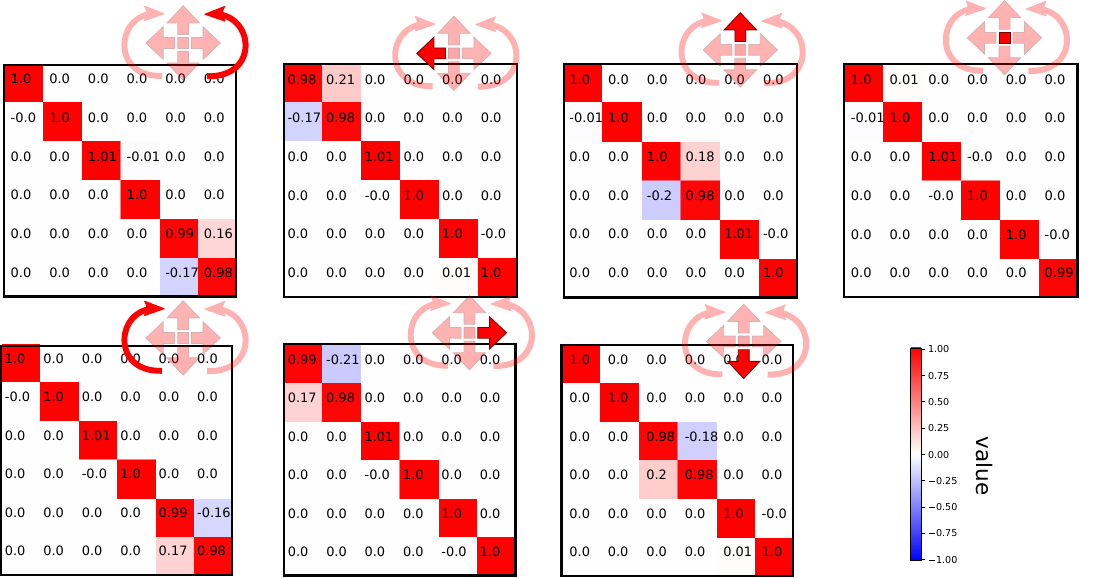}
    \caption{Evaluation of the learned group representation for the heart dataset. The identity (upper left) and generative transitions of each subgroup yields disentangled block rotation matrices.}
    \label{fig:repr_xytheta}
\end{figure}

\subsection{Rollouts prediction}\label{app:exp-rollout}
We adapt the setting of multi-step prediction as in \cite{Quessard2020}, where agents can perform multiple simple actions to the object.
In our experiment, we allow the agent to control the object in the dSprite dataset with 7 actions.
Namely, translation in the x-y axes, rotation in both directions (clockwise, counter-clockwise), and idle.
Each action corresponds to an increment/decrement in one of the generating factors of the dataset, except for idle, which does nothing.
Additionally, we use the heart shape from the dataset to fully utilize the orientation latent factor.

We train each method with a set of pre-generated set of 2-step trajectories and evaluate on a hold-out pre-generated set of 128-step trajectories.
For each trajectory, we begin by sampling a random initial state (x,y position and orientation) from all possible states.
After that, we sample actions uniformly from the 7 possible actions at each step until the number of steps is satisfied.

For the Rotations method \citep{Quessard2020} and HAE, each action is represented using a matrix $\rho(g)$, and the transition in the latent space is simply $z_{t+1} = \rho(g_t)z_t$.
For the Unstructured method, we use a 2 layer MLP of size [128, 128] to model the transition by $z_{t+1} = f_\theta(z_t, g_t)$, where we concatenate the latent vector $z_t$ and the one-hot encoding of the action $a_t$.

The reconstruction loss is the same for all three methods, as described in Section~\ref{sec:theory}.
For HAE, we additionally add the latent prediction loss $\mathcal{L}_{pred}$ as described in Section~\ref{sec:theory}.
We increase $\gamma$ to 1600 which we found to be more stable when matrices are directly parameterized instead of mapped from MLPs.
For the Rotations method, an additional entanglement loss $\mathcal{L}_{ent}$ is required to encourage each matrix to act on the fewest dimensions of the latent space, which is equal to
\[
    \mathcal{L}_{ent} = \sum_g\sum_{(i,j)\neq(\alpha,\beta)} |\theta_{i,j}^g|^2 \quad \text{with}\quad \theta^g_{\alpha,\beta} =\max_{i,j} |\theta^g_{i,j}|.
\]
For the Unstructured method, we only use the reconstruction loss and no additional terms.

Aside from the action sampling scheme described above, we also perform another rollout experiment using the sampling strategy described in Section~\ref{sec:sampling} where actions were sampled around the identity uniformly. 
We use $\varphi=id$. 
For simplicity, we reduce the range of the actions from $\llbracket -10,10 \rrbracket$ to $\llbracket -3,3 \rrbracket$, also, we only consider x-y translation in this additional experiment.
The training process is the same as described above, where we train and test on offline pre-generated datasets of 2-step and 128-step trajectories.
\begin{figure}[ht]
    \centering
    \includegraphics[width=0.6\textwidth]{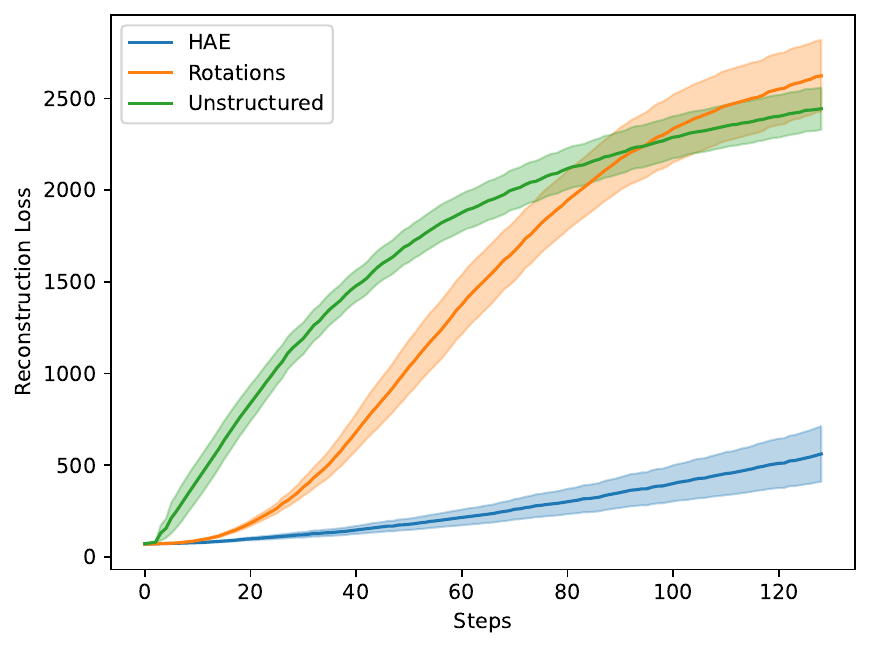}
    \caption{Step-wise reconstruction loss for the $SO(2)\times SO(2) \times SO(2)$ experiment. Lines and shadings represent the mean and one standard error over 15 seeds.}
    \label{fig:recon_unif}
\end{figure}

\subsection{Multi-objects}\label{app:exp-multiobjects}
We use a similar dataset to the one described in section \ref{app:exp-single-shape} except that we use all three shapes: Heart, Ellipse and Square. Because we are considering the action of the same group, observation sequences that start with a given shape have the same shape throughout at different positions. We use a HAE model with a representation of dimension $8$, trained for disentanglement using the soft block sparsity regularization described in section~\ref{sec:soft-dis}. The model learns representations of the observations that span $5$ dimensions visualized in Figure~\ref{fig:muliobject-e}, where $4$ representation units contain information about pose with regard to the translation group, while the remaining $4$ dimensions encode shape along a single principal component.
The associated learned group representation $\rho$ can be visualized for a neighbourhood of the identity in Figure~\ref{fig:rho-multishape}.


\begin{figure}[ht]
    \centering
    \includegraphics[width=0.6\linewidth]{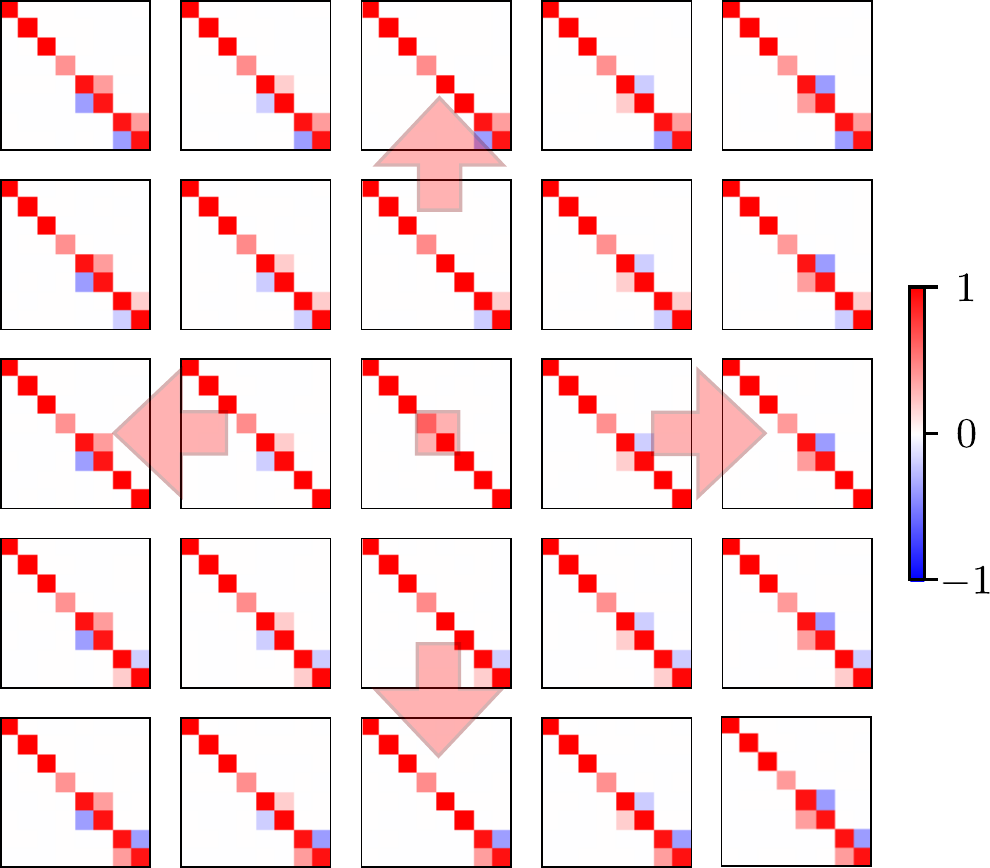}
    \caption{HAE trained on multiple dSprites  shapes. Visualization of the group representation matrices $\rho(\varphi(g))$ for $g$ in a square around the identity. The representation of the identity transition $g=0$ is found at the center of the grid and red arrows indicate the direction of the latent transition signal $g$. $\rho$ exhibits disentangled block diagonal structure and each non-trivial block shows the expected rotation matrices predicted in Equation~\ref{eq:expectedrho}}.
    \label{fig:rho-multishape}
\end{figure}


\section{Experiment: Learning representations structured by $G=SO(3)\times SO(2)$}\label{app:exp-bunny}
\subsection{Data generation}
To generate the transitions dataset ${(o_t,g_t,o_{t+1},g_{t+1},o_{t+2})}_t$, we apply small relative rotations to the vertices of the bunny.obj and render it with $20$ colors equally spaced from the color hue wheel. The interaction $g_t$ decomposes into the relative rotation angles (Roll, Pitch, Yaw) and the color shift on the color wheel. The rotations are sampled from a continuous uniform distribution centered around the identity and maximum angles of $0.5 rad$, the color shifts are simultaneously sampled from a discrete uniform distribution centered around $0$ and with a maximum color shift of $3$. The dataset we used contains $300K$ such interaction sequences.

\subsection{Hyperparameters}
\paragraph{Model architecture} We use a symmetrical architecture for the encoder and the decoder, summarized in Table~\ref{tab:network-bun}. The network was trained on the combined loss:
\[
\mathcal{L} = \mathcal{L}_{rec}^2(\rho,h) + \gamma \mathcal{L}_{pred}^2(\rho,h)
\]
We use the Binary Cross Entropy loss for the reconstruction term instead of the Mean Squared Error.

\begin{table}[h!]
    \centering
    \caption{Network architecture.}
    \begin{tabular}{cc}
        Parameter & Value \\ \midrule
        Conv. Channels & [64, 64, 64, 64] \\
        Kernel Sizes & [6, 4, 4, 4] \\
        Strides & [2, 2, 1, 1] \\
        Linear Layer Size & 1024 \\
        Activation & ReLU \\
        Latent space & 11\\
        $\gamma$ & 200\\
        Group representation dimensions & [2,3,3,3]
    \end{tabular}
    \label{tab:network-bun}
\end{table}

\paragraph{Training hyperparameters}
We trained the network using the hyperparameters summarized in Table~\ref{tab:hp-bun}.
\begin{table}[h!]
    \centering
    \caption{Training hyperparameters.}
    \begin{tabular}{cc}
        Parameter & Value \\ \midrule
        Optimizer & Adam \\
        Learning rate & 0.0005 \\
        Number of training sequences & 300000 \\
        Batch size & 300 \\
        Epochs & 501 \\
        Iterations per epoch & 600\\
    \end{tabular}
    \label{tab:hp-bun}
\end{table}

\subsection{Rollouts prediction}\label{app:exp-rollout-bunny}
Similar to dSprites, we also test our method's ability to predict longer rollouts than the $2-$steps rollouts encountered during training.
For simplicity, we only consider the group  $G=SO(3)$ of rotations without color shifts.
We use the same network architecture described in Table~\ref{tab:network-bun}.
The training set consists of 300,000 2-step trajectories, while the test set consists of 200 128-step trajectories.
We compare our method to the Unstructured method described in Section~\ref{app:exp-rollout}.

\section{Third-Party Software}
\subsection{Deep Learning Framework}
To implement our architecture we used the deep learning framework PyTorch. \cite{Ansel_PyTorch_2_Faster_2024}

\subsection{Hyperparameter Search}
We used the hyperparameter search utility provided in the hypnettorch project \url{https://github.com/chrhenning/hypnettorch/tree/master/hypnettorch/hpsearch} to perform a random grid search.

\subsection{Dataset}
In the presented experiments, we used the dSprites dataset \cite{dsprites17}. The dSprites dataset is an image dataset of white sprites on a black background, varying in shape (heart, ellipse, square), in scale ($6$ values), in orientation ($39$ values, cyclic), in $x$ and $y$ position ($32$ values each). We consider all factors besides shape to be cyclic, in particular for the $x$ and $y$ positions, we "glued" opposite borders of images into a torus. The resolution of the images is $64\times 64$ pixels.

For the 3D rotation experiment, we used the Stanford bunny \cite{turk1994zippered}, which was obtained from 3D scanning a ceramic figurine of a rabbit. The 3D model consists of 35947 vertices and 69451 triangles, 
for a more detailed description, please see \url{http://graphics.stanford.edu/data/3Dscanrep/}. The model is colored from $10$ equally spaced colors on the hue wheel. We render colored images of shape $[3,72,72]$.
We act on the $3D$ model by applying small rotations and small translations of its color on the hue wheel.

\section{Computational resources}
The experiments were performed on an NVIDIA GeForce RTX 3090 and A100 GPUs, and training a model takes approximately $20$ mins for the dSprites dataset and $3-5$ hours for the $3D$ dataset.

\section{Societal Impact}
This work proposes new findings in basic research. 
To the best of our knowledge, this work does not have immediate applications with a negative societal impact.


\end{document}